\documentclass[11pt]{article}

\usepackage{thmtools}
\usepackage{thm-restate}
\usepackage{bm}
\usepackage{multirow}
\usepackage{mathrsfs}           
\usepackage{tikz}
\usetikzlibrary{positioning,chains,fit,shapes,calc}
\usepackage{algorithm}
\usepackage{algpseudocode}
\usepackage{mathtools}

\usepackage{subcaption}
\usepackage{enumerate}

\usepackage{amsmath,amssymb, amsthm}    
\usepackage{verbatim}           
\usepackage{xspace}             
\usepackage{graphicx,float}     
\usepackage{ifthen,calc}        
\usepackage{textcomp}           
\usepackage{fancybox}           
\usepackage{hhline}             
\usepackage{float}              

\usepackage[vflt]{floatflt}     
\usepackage[small,compact]{titlesec}
\usepackage{setspace}

\usepackage{enumitem}
\setenumerate[1]{label={(\arabic*)}}

\usepackage{color}
\definecolor{ForestGreen}{rgb}{0.1333,0.5451,0.1333}
\usepackage{thumbpdf}
\usepackage[letterpaper,
colorlinks,linkcolor=ForestGreen,citecolor=ForestGreen,
backref,
bookmarks,bookmarksopen,bookmarksnumbered]
{hyperref}
\usepackage{cleveref}
\usepackage[margin = 1in]{geometry}
\newcommand{\showccc}[0]{0}
\newcommand{\ccc}[2][nothing]{
	\ifthenelse{\showccc=0}{}{
		\ensuremath{^{\Lsh\Rsh}}\marginpar{\raggedright\tiny\textsf{%
				\ifthenelse{\equal{#1}{nothing}}{}{\textbf{#1}\\}#2}}}}
\newtheorem{theorem}{Theorem}

\newtheorem{proposition}{Proposition}
\newtheorem{corollary}{Corollary}
\newtheorem{definition}{Definition}

\newtheorem{lemma}{Lemma}
\newtheorem{fact}{Fact}

\newtheorem{assumption}{Assumption}

\newcommand{\defeq}{:=}
\newcommand{\norm}[1]{\left\lVert#1\right\rVert}

\newcommand{\eps}{\epsilon}
\newcommand{\lam}{\lambda}

\newcommand{\argmin}{\textup{argmin}} 
\newcommand{\R}{\mathbb{R}}

\newcommand{\N}{\mathbb{N}}

\newcommand{\half}{\frac{1}{2}}

\newcommand{\E}{\mathbb{E}}

\newcommand{\Nor}{\mathcal{N}}

\newcommand{\xset}{\mathcal{X}}

\definecolor{burntorange}{rgb}{0.8, 0.33, 0.0}
\newcommand{\kjtian}[1]{\textcolor{burntorange}{\textbf{kjtian:} #1}}
\newcommand{\Daogao}[1]{\textcolor{blue}{\textbf{Daogao:} #1}}
\newcommand{\tO}{\widetilde{O}}

\newcommand{\Par}[1]{\left(#1\right)}
\newcommand{\Brack}[1]{\left[#1\right]}

\newcommand{\Abs}[1]{\left|#1\right|}

\newcommand{\tvd}[2]{\norm{#1 - #2}_{\textup{TV}}}

\definecolor{ruoqi}{rgb}{0.2, 0.3, 0.7}

\newcommand{\tpi}{\tilde{\pi}}

\newcommand{\normx}[1]{\norm{#1}_{\xset}}

\newcommand{\dkl}{D_{\textup{KL}}}

\newcommand{\data}{\mathcal{D}}
\newcommand{\mech}{\mathcal{M}}
\newcommand{\tF}{\widetilde{F}}
\newcommand{\tf}{\tilde{f}}
\newcommand{\msig}{\boldsymbol{\Sigma}}
\newcommand{\set}{\mathcal{S}}
\newcommand{\Fpop}{F_{\textup{pop}}}
\newcommand{\mm}{\mathbf{M}}
\newcommand{\diam}{\textup{diam}}
\newcommand{\calP}{\mathcal{P}}

\newcommand{\conv}{\mathrm{conv}}
\newcommand{\Ind}{\mathbf{1}}
\DeclarePairedDelimiterX{\xdivergence}[2]{(}{)}{%
	#1\;\delimsize\|\;#2%
}
\newcommand{\deltacurve}{\delta\xdivergence*}

\newcommand{\lp}{\left(}

\newcommand{\rp}{\right)}

\usepackage{float}
\floatstyle{plain}\newfloat{myfig}{t}{figs}[section]
\floatname{myfig}{\textsc{Figure}}
\floatstyle{plain}\newfloat{myalg}{H}{algs}[section]
\floatname{myalg}{}
\usepackage{url}

\begin{document}

	\begin{titlepage}
		\def\thepage{}
		\thispagestyle{empty}
		
		\title{Private Convex Optimization in General Norms} 
		
		\date{}
		\author{
			Sivakanth Gopi\thanks{Microsoft Research, {\tt sigopi@microsoft.com}}
			\and
			Yin Tat Lee\thanks{University of Washington, {\tt yintat@uw.edu}}
			\and
			Daogao Liu\thanks{University of Washington, {\tt dgliu@uw.edu}}
			\and
			Ruoqi Shen\thanks{University of Washington, {\tt shenr3@cs.washington.edu}}
			\and
			Kevin Tian\thanks{Microsoft Research, {\tt tiankevin@microsoft.com}. Work completed while at Stanford University.}
		}
		
		\maketitle
		
		\abstract{
We propose a new framework for differentially private optimization of convex functions which are Lipschitz in an arbitrary norm $\normx{\cdot}$. Our algorithms are based on a regularized exponential mechanism which samples from the density $\propto \exp(-k(F+\mu r))$ where $F$ is the empirical loss and $r$ is a regularizer which is strongly convex with respect to $\normx{\cdot}$, generalizing a recent work of \cite{GLL22} to non-Euclidean settings. We show that this mechanism satisfies Gaussian differential privacy and solves both DP-ERM (empirical risk minimization) and DP-SCO (stochastic convex optimization) by using localization tools from convex geometry. Our framework is the first to apply to private convex optimization in general normed spaces, and directly recovers non-private SCO rates achieved by mirror descent, as the privacy parameter $\eps \to \infty$. As applications, for Lipschitz optimization in $\ell_p$ norms for all $p \in (1, 2)$, we obtain the first optimal privacy-utility tradeoffs; for $p = 1$, we improve tradeoffs obtained by the recent works \cite{AsiFKT21, BassilyGN21} by at least a logarithmic factor. Our $\ell_p$ norm and Schatten-$p$ norm optimization frameworks are complemented with polynomial-time samplers whose query complexity we explicitly bound.
}
		
	\end{titlepage}


\section{Introduction}
\label{sec:intro}

The study of convex optimization in spaces where the natural geometry is non-Euclidean, beyond being a natural question of independent interest, has resulted in many successes across algorithm design. A basic example of this is the celebrated multiplicative weights, or exponentiated gradient method \cite{AroraHK12}, which caters to the $\ell_1$ geometry and has numerous applications in learning theory and algorithms. Moreover, optimization in real vector spaces equipped with different $\ell_p$ norms has found use in sparse recovery \cite{CandesRT06}, combinatorial optimization \cite{KelnerLOS14, KyngPSW19}, multi-armed bandit problems \cite{BubeckC12}, fair resource allocation \cite{DiakonikolasFO20}, and more (see e.g.\ \cite{AdilKPS19, DiakonikolasG21} and references therein). Furthermore, optimization in Schatten-$p$ norm geometries (the natural generalization of $\ell_p$ norms to matrix spaces) has resulted in improved algorithms for matrix completion \cite{AgarwalNW10} and outlier-robust PCA \cite{Jambulapati0T20}. In addition to $\ell_p$ and Schatten-$p$ norms, the theory of non-Euclidean geometries has been very useful in settings such as linear and semidefinite programming \cite{nemirovski2004interior} and optimization on matrix spaces \cite{Allen-ZhuGLOW18}, amongst others. 

The main result of this paper is a framework for \emph{differentially private} convex optimization in general normed spaces under a Lipschitz parameter bound. Differential privacy \cite{DworkKMMN06, DworkMNS06} has been adopted as the standard privacy notion for data analysis in both theory and practice, and differentially private algorithms have been deployed in many important settings in the industry as well as the U.S.\ census \cite{ErlingssonPK14, Abo16, App17, BittauEMMRLRKTS17, DingKY17}. Consequently, differentially private optimization is an increasingly important and fundamental primitive in modern machine learning applications \cite{BassilyST14, AbadiCGMMT016}. However, despite an extensive body of theoretical work providing privacy-utility tradeoffs (and more) for optimization in the Euclidean norm geometry, e.g.\ \cite{ChaudhuriM08, ChaudhuriMS11, KiferST12, JainT14, BassilyST14, Kasiviswanathan16} (and many other follow-up works), more general settings have been left relatively unexplored. This state of affairs prohibits the application of private optimization theory to problems where the natural geometry is non-Euclidean. Recent works \cite{AsiFKT21, BassilyGN21, BassilyGM21} have investigated special cases of private convex optimization, e.g.\ for $\ell_p$ spaces or polyhedral sets, under smoothness assumptions, or under structured losses. However, the systematic study of private convex optimization in general normed spaces in the most fundamental setting of Lipschitz losses has been left open, a gap that our work addresses.

Our framework for private convex optimization is simple: we demonstrate strong privacy-utility tradeoffs for a \emph{regularized exponential mechanism} when optimizing a loss over a set $\xset \subset \R^d$ equipped with a norm $\normx{\cdot}$. More concretely, our algorithms sample from densities 
\[\propto \exp\Par{-k(F_{\data} + \mu r)}\]
where $k, \mu > 0$ are tunable parameters, $F_{\data}$ is a (data-dependent) empirical risk, and $r$ is a strongly convex regularizer in $\normx{\cdot}$ with bounded range over $\xset$. In the analogous non-private Lipschitz convex optimization setting, most theoretical developments (namely mirror descent frameworks) have focused on precise applications where such an $r$ is readily available \cite{Shalev-Shwartz12, Bubeck15}. In this sense, our framework directly extends existing Lipschitz convex optimization theory to the private setting (and indeed, recovers existing non-private guarantees obtained by mirror descent \cite{NemirovskiY83}). 

In the remainder of the introduction, we summarize our results (Section~\ref{ssec:results}), highlight our technical contributions (Section~\ref{ssec:techniques}), and situate our paper in the context of prior work (Section~\ref{ssec:prior}).

\subsection{Our results}\label{ssec:results}

We study both the empirical risk minimization (ERM) problem and the stochastic convex optimization (SCO) problem in this paper; the goal in the latter case is to minimize the \emph{population risk}. We formalize this under the following assumption, which parameterizes the space we are optimizing and the (empirical and population) risks we aim to minimize. 

\begin{assumption}\label{assume:norm}
	We make the following assumptions.
	\begin{enumerate}
		\item There is a compact, convex set $\xset \subset \R^d$ equipped with a norm $\normx{\cdot}$.
		\item There is a $1$-strongly convex function $r: \xset \to \R$ in $\normx{\cdot}$, and $\Theta \ge \max_{x \in \xset} r(x) - \min_{x \in \xset} r(x)$. 
		\item There is a set $\Omega$ such that for any $s \in \Omega$, there is a convex function $f(\cdot; s): \xset \to \R$ which is $G$-Lipschitz in $\normx{\cdot}$.
	\end{enumerate}
\end{assumption}

For definitions used above, see Section~\ref{sec:prelims}. We remark that by strong convexity, the parameter $\Theta$ scales at least as $\Omega(D^2)$, where $D$ is the diameter of $\xset$ with respect to $\normx{\cdot}$; in many cases of interest, we may upper bound $\Theta$ by $O(D^2)$ as well up to a logarithmic factor. 

Finally, throughout the paper when working under Assumption~\ref{assume:norm}, $\data = \{s_i\}_{i \in [n]}$ denotes a dataset drawn independently from $\calP$, a distribution supported on $\Omega$, and we define $F_{\data}: \xset \to \R$ and $\Fpop: \xset \to \R$ by
\begin{equation}\label{eq:Fdef} F_{\data}(x) \defeq \frac 1 n \sum_{i \in [n]} f(x; s_i),\; \Fpop(x) \defeq \E_{s \sim \calP}[f(x; s)].\end{equation}

\paragraph{Private ERM and SCO.} We first present the following general results under Assumption~\ref{assume:norm}.

\begin{theorem}[Informal, see Theorems~\ref{thm:erm},~\ref{thm:sco}]\label{thm:informal}
Under Assumption~\ref{assume:norm} and following notation \eqref{eq:Fdef}, drawing a sample $x$ from the density $\nu \propto \exp(-k(F_{\data} + \mu r))$ for some $k, \mu > 0$ specified in Theorem~\ref{thm:erm} is $(\eps, \delta)$-differentially private, and produces $x$ such that
\[\E_{x \sim \nu}[F_{\data}(x)] - \min_{x \in \xset} F_{\data}(x) \le G\sqrt{\Theta} \cdot \frac{\sqrt{8d\log \frac 1 {2\delta}}}{n\eps}.\]
Moreover, drawing a sample $x$ from the density $\nu \propto \exp(-k(F_{\data} + \mu r))$ for some $k, \mu > 0$ specified in Theorem~\ref{thm:sco} is $(\eps, \delta)$-differentially private, and produces $x$ such that
\[\E_{\data \sim \calP^n, x \sim \nu}\Brack{\Fpop(x)} - \min_{x \in \xset} \Fpop(x) \le G\sqrt{\Theta} \cdot \Par{\frac{\sqrt{8d\log \frac 1 {2\delta}}}{n\eps} + \sqrt{\frac 8 n}}.\]
\end{theorem}
Minimizing the non-private population risk under the same setting as Assumption~\ref{assume:norm} is a very well-studied problem, with matching upper and lower bounds in many cases of interest, such as $\ell_p$ norms \cite{NemirovskiY83, AgarwalBRW12, DuchiJW14}. The population utility achieved by our regularized exponential mechanism in Theorem~\ref{thm:informal} (namely, as $\eps \to \infty$) matches the rate obtained by the classical mirror descent algorithm \cite{NemirovskiY83}, which to our knowledge has not been previously observed. Finally, in Appendix~\ref{app:sc} we provide an analog of Theorem~\ref{thm:informal} under the stronger assumption that the sample losses $f(\cdot; s)$ are strongly convex, bypassing the need for explicit regularization. Our results in Appendix~\ref{app:sc} recover the optimal rate in the Euclidean case, matching known lower bounds \cite{BassilyST14}.

We next show how to apply the results of Theorem~\ref{thm:informal} under various instantiations of Assumption~\ref{assume:norm} to derive new rates for private convex optimization under $\ell_p$ and Schatten-$p$ norm geometries.

\paragraph{$\ell_p$ and Schatten-$p$ norms.} In Corollaries~\ref{cor:pg1le2},~\ref{cor:pge1le2}, and~\ref{cor:pge2}, we combine known (optimal) uniform convexity estimates for $\ell_p$ spaces \cite{BallCL94} with the algorithms of Theorem~\ref{thm:erm} and~\ref{thm:sco} to obtain privacy-utility tradeoffs summarized in Table~\ref{table:lp}. Interestingly, we achieve all these bounds with a single algorithmic framework, which in all cases matches or partially matches known lower bounds.

\begin{table}[ht!]
	\begin{centering}
		\begin{tabular}{|c|c|c|}
			\hline 
			\multirow{2}{*}{$\ell_p$ norm} & \multicolumn{2}{c|}{Optimality gap}\tabularnewline
			\cline{2-3} \cline{3-3} 
			& ERM loss $F_{\data}$ &  SCO loss $\Fpop$ \tabularnewline
			\hline 
			$p \in (1, 2)$ $(\star)$ & $GD \cdot \frac{\sqrt{d \log \frac 1 {2\delta}}}{n\eps\sqrt{p - 1}}$ & $GD \cdot \Par{\frac 1 {\sqrt{n(p - 1)}} + \frac{\sqrt{d \log \frac 1 {2\delta}}}{n\eps\sqrt{p - 1}}}$ \tabularnewline
			\hline 
			$p = 1$ $(\dagger)$ & $GD \cdot \frac{\sqrt{d \log d \log \frac 1 {2\delta}}}{n\eps}$ & $GD \cdot \Par{\sqrt{\frac {\log d} n} + \frac{\sqrt{d \log d\log \frac 1 {2\delta}}}{n\eps}}$\tabularnewline
			\hline 
			$ p \ge 2$ $(\dagger)$ & $GD \cdot \frac{d^{1 - \frac 1 p}\sqrt{\log \frac 1 {2\delta}}}{n\eps}$ & $GD \cdot\Par{\frac{d^{\half - \frac 1 p}}{\sqrt n} + \frac{d^{1 - \frac 1 p}\sqrt{\log \frac 1 {2\delta}}}{n\eps}}$ \tabularnewline
			\hline 
		\end{tabular}
		\par\end{centering}
	\caption{\label{table:lp}Privacy-utility tradeoffs for $\ell_p$ norm optimization under $(\eps, \delta)$-differential privacy obtained by: Corollary~\ref{cor:pg1le2} ($p \in (1, 2)$), Corollary~\ref{cor:pge1le2} ($p = 1$), and Corollary~\ref{cor:pge2} ($p \ge 2$). We assume $\xset$ has $\ell_p$ diameter bounded by $D$ and hide constants (stated in the formal results) for brevity. $(\star)$ indicates that our result matches the private ERM and SCO lower bound \cite{BassilyGN21, LiuL22}. $(\dagger)$ indicates that our result (as $\eps \to \infty$) matches the non-private SCO lower bound \cite{AgarwalBRW12, DuchiJW14}.}
\end{table}

We now contextualize our results with regard to the existing literature. In the following discussion, the ``privacy-dependent'' loss term is the term in the SCO loss scaling with $\eps, \delta$, and the ``privacy-independent'' loss term is the SCO loss when $\eps \to \infty$.

In the case of constant $p \in (1, 2)$, our Corollary~\ref{cor:pg1le2} sharpens Theorem 5 of \cite{AsiFKT21} by a $\sqrt{\log d}$ factor in the privacy-dependent loss term, and is the first to match lower bounds of \cite{BassilyGN21, LiuL22}. It improves bounds by \cite{BassilyGN21} by at least a $\log n$ factor on both parts of the SCO loss, which further loses an $n^{\frac 1 4}$ factor on the privacy-dependent loss and requires additional smoothness assumptions. 

In the important case of $p = 1$, of fundamental interest due to its applications in sparse recovery \cite{CandesRT06} as well as online learning \cite{Shalev-Shwartz12, AroraHK12}, our Corollary~\ref{cor:pge1le2} improves the privacy-dependent loss term of \cite{AsiFKT21} by a $\log d$ factor, and matches the privacy-independent loss lower bound in the SCO literature \cite{DuchiJW14}, matching the rate of entropic mirror descent. The privacy-dependent loss term incurs an additional overhead of $\sqrt{\log d}$ compared to existing lower bounds. However, just as lower bounds on the privacy-independent loss increase as $p \to 1$, it is plausible that the upper bound obtained by Corollary~\ref{cor:pge1le2} is optimal, which we leave as an interesting open direction.

In the $p \ge 2$ case, prior work by \cite{BassilyGN21} obtains a rate matched by Corollary~\ref{cor:pge2}. The non-private population risk term in \eqref{eq:pge2sco} is again known to be optimal \cite{AgarwalBRW12}. We again find it an interesting open direction to close the gap between the upper bound \eqref{eq:pge2erm} and known lower bounds for private convex optimization when $p \ge 2$, e.g.\ \cite{BassilyGN21, LiuL22}.

We further demonstrate in Corollary~\ref{cor:schatten} that all of these results have direct analogs in the case of optimization over matrix spaces equipped with Schatten norm geometries. To the best of our knowledge, this is the first such result in the private optimization literature; we believe this showcases the generality and versatility of our approach.

Finally, we mention that all of these results are algorithmic and achieved by samplers with polynomial query complexity and runtime, following developments of \cite{lee2021structured, GLL22}. In all cases, by simple norm equivalence relations, the query complexity of our samplers is at most a $d$ factor worse than the $\ell_2$ case, with improvements as $p \to 2$. It is an exciting direction to develop efficient high-accuracy samplers catering to structured densities relevant to the setups considered in this paper, e.g.\ those whose negative log-likelihoods are strongly convex in $\ell_p$ norms. The design of sampling algorithms for continuous distributions has been an area of intense research activity in the machine learning community, discussed in greater detail in Section~\ref{ssec:prior}. We mention here that our hope is that our results and general optimization framework serve as additional motivation for the pursuit of efficient structured sampling algorithms working directly in non-Euclidean geometries.

\subsection{Our techniques}\label{ssec:techniques}

Our results essentially build on the recent work of \cite{GLL22}, who observed that a regularized exponential mechanism achieves optimal privacy-utility tradeoffs for empirical and population risks when losses are Lipschitz in the $\ell_2$ norm. Under a Euclidean specialization of Assumption~\ref{assume:norm}, \cite{GLL22} provided variants of Theorem~\ref{thm:informal} using the regularizer $r(x) = \half \norm{x}_2^2$, i.e.\ reweighting by a Gaussian. 

We demonstrate several key tools used in \cite{GLL22} have non-Euclidean extensions by using a simple, general approach based on a convex geometric tool known as \emph{localization}. For example, the starting point of our developments is relating the privacy curves of two nearby, strongly convex densities with the privacy curve of Gaussians (see Section~\ref{sec:prelims} for definitions).

\begin{restatable}{theorem}{restategdp}
	\label{thm:privacy_curve_norm_X}
	Let $\xset \subset \R^d$ be compact and convex, let $F, \tF: \xset \to \R$ be $\mu$-strongly convex in $\normx{\cdot}$, and let $P \propto \exp(-F)$ and $Q \propto \exp(-\tF)$. Suppose $\tF - F$ is $G$-Lipschitz in $\normx{\cdot}$. For all $\eps \in \R_{\ge 0}$,
	\[\deltacurve{P}{Q}(\eps) \le \deltacurve{\Nor(0, 1)}{\Nor\Par{\frac G {\sqrt \mu}, 1}}(\eps).\]
\end{restatable}

An analog of Theorem~\ref{thm:privacy_curve_norm_X} when $\normx{\cdot}$ is the Euclidean norm was proven as Theorem 4.1 of \cite{GLL22}. Moreover, the analog of Theorem~\ref{thm:informal} in \cite{GLL22} follows from combining Theorem 4.1 of that work, and their Theorem 6.10, a reduction from the SCO problem to the ERM problem (containing a generalization error bound). These proofs in \cite{GLL22} rely on powerful inequalities from probability theory, which were initially studied in the Gaussian (Euclidean norm regularizer) setting. For example, Theorem 4.1 applied the \emph{Gaussian isoperimetric inequality} of \cite{SudakovT74, Borell75} (see also Theorem 1.1, \cite{Ledoux99}), which states that strongly logconcave distributions in the Euclidean norm have expansion quality at least as good as a corresponding Gaussian. Moreover, the generalization error bound in Theorem 6.10 was proven based on a Euclidean norm \emph{log-Sobolev inequality} and \emph{transportation inequality}, relating Wasserstein distances, KL divergences, and Lipschitz bounds on negative log-densities. Fortunately, it turns out that all of these inequalities have non-Euclidean generalizations (possibly losing constant factors). For example, a non-Euclidean log-Sobolev inequality was shown by Proposition 3.1 of \cite{BobkovL00}, and a non-Euclidean transport inequality sufficient for our purposes is proved as Proposition 1 of \cite{Cordero17}. Finally, variants of the Gaussian isoperimetric inequality in general norms are given by \cite{MilmanS08, Kolesnikov11}. Plugging in these tools into the proofs of \cite{GLL22} allows us to recover Theorems~\ref{thm:informal} and~\ref{thm:privacy_curve_norm_X}, as well as our applications.

In this work, we take a different (and in our opinion, simpler) strategy to proving the probability-theoretic inequalities required by Theorems~\ref{thm:informal},~\ref{thm:privacy_curve_norm_X}, yielding an alternative to the proofs in \cite{GLL22} which we believe may be of independent intellectual interest to the privacy community. In particular, our technical insight is the simple observation that several of the definitions in differential privacy are naturally cast in the language of localization \cite{kannan1995isoperimetric, FG04}, which characterizes extremal logconcave densities subject to linear constraints (see our proof of Lemma~\ref{lm:reducet_onedim_privacy} for an example of this). This observation allows us to reduce the proofs of key technical tools used in Theorems~\ref{thm:informal} and~\ref{thm:privacy_curve_norm_X} to proving these tools in one dimension, where \emph{all norms are equivalent} up to constant factor rescaling.\footnote{The one-dimensional case can then typically be handled by more straightforward ``combinatorial'' arguments, see e.g.\ Section 2.b of \cite{lovasz1993random} or Appendix B.3 of \cite{ChenDWY20} for examples.} After deriving several extensions of basic localization arguments in Section~\ref{ssec:localization}, we follow this reduction approach to give a more unified proof to Theorems~\ref{thm:informal} and~\ref{thm:privacy_curve_norm_X}. To our knowledge, this is the first direct application of localization techniques in differential privacy. 

The interplay between the privacy and probability theory communities is an increasingly active area of exploration \cite{DongRS21, GLL22, GaneshTU22} (discussed in more detail in Section~\ref{ssec:prior}). We are hence optimistic that localization-based proof strategies will have further applications in the privacy literature, especially in situations (beyond this paper) where probability theoretic tools used in the Euclidean case do not have non-Euclidean variants in full generality. In such settings, it may be a valuable endeavor to see if necessary inequalities may be directly recast in the language of localization.

\subsection{Prior work}\label{ssec:prior}

\paragraph{Private optimization in Euclidean norm.} Many prior works on private convex optimization have focused on variants of the ERM and SCO problems studied in this work, under $\ell_2$ Lipschitz losses and $\ell_2$ bounded domains, such as \cite{ChaudhuriMS11, KiferST12, BassilyST14, BassilyFTT19, BassilyFGT20}. The optimal information-theoretic rate for these private optimization problems was given by \cite{BassilyST14}, which was matched algorithmically up to constant factors by \cite{BassilyFTT19, BassilyFGT20}. 

From an algorithmic perspective, a topic of recent interest in the Euclidean case is the problem of attaining optimal privacy-utility tradeoffs in \emph{nearly-linear time}, namely, with $\approx n$ gradient queries \cite{FeldmanKT20, AsiFKT21, KulkarniLL21}. Under additional smoothness assumptions, this goal was achieved by \cite{FeldmanKT20}; however, achieving near-optimal gradient oracle query rates in the general Lipschitz case remains open. We note that under \emph{value oracle} access, a near-optimal bound was recently achieved by \cite{GLL22}. This paper primarily focuses on the information-theoretic problem of achieving optimal privacy-utility tradeoffs for a given dataset size. However, we believe the corresponding problem of designing algorithms with near-optimal query complexities and runtimes (under value or gradient oracle access) is also an important open direction in the case of general norm geometries.

\paragraph{Private optimization in non-Euclidean norms.} The study of convex optimization in non-Euclidean geometries was recently initiated by \cite{AsiFKT21, BassilyGN21}, who focused primarily on developing algorithms under $\ell_p$ regularity assumptions and bounded domains. In follow-up work, \cite{BassilyGM21} gave improved guarantees for the family of generalized linear losses. We discuss the rates we achieve for $\ell_p$ norm geometries compared to \cite{AsiFKT21, BassilyGN21} in Section~\ref{ssec:results}; in short, we improve prior results by logarithmic factors in the case $p \in [1, 2)$, and match them when $p \ge 2$. Independently from our work, \cite{HLL+22} designed an algorithm for private optimization in $\ell_p$ geometries improving upon \cite{BassilyGN21} in gradient query complexity (matching their privacy-utility tradeoffs); both \cite{BassilyGN21, HLL+22} require further smoothness assumptions on the loss functions.

One of the main motivations for this work was to develop a general theory for private convex optimization under non-Euclidean geometries, beyond $\ell_p$ setups. In particular, \cite{BassilyGN21} designed a \emph{generalized Gaussian mechanism} for the case $p \in [1, 2)$, where gradients were perturbed by a noise distribution catering to the $\ell_p$ geometry. However, how to design a corresponding mechanism for more general norms may be less clear. The algorithm of \cite{AsiFKT21} in the non-smooth case was based on a (Euclidean norm) Gaussian mechanism; again, this strategy is potentially more specialized to $\ell_p$ geometries. Beyond giving a general algorithmic framework for non-Euclidean convex optimization based on structured logconcave sampling, we hope that the information-theoretic properties we show regarding regularized exponential mechanisms (e.g.\ Theorem~\ref{thm:privacy_curve_norm_X}) may find use in designing ``generalized Gaussian mechanisms'' beyond $\ell_p$ norms.

\paragraph{Connections between privacy and sampling.} Our work extends a line of work exploring privacy-utility tradeoffs for the exponential mechanism, a general strategy for designing private algorithms introduced by \cite{McSherryT07} (see additional discussion in \cite{GLL22}). For example, the regularized exponential mechanisms we design are similar in spirit to the exponential mechanism ``in the $\xset$ norm\footnote{That is, the norm induced by the convex body $\xset$, not to be confused with the $\normx{\cdot}$ of Assumption~\ref{assume:norm}.}'' designed by \cite{HT10, BDKT12}. Moreover, our work continues a recent interface between the sampling and privacy literature, where (continuous and discrete-time) sampling algorithms are shown to efficiently obtain strong privacy-utility tradeoffs for optimization problems \cite{GLL22, GaneshTU22}. This work further develops this interface, motivating the design of efficient samplers for densities satisfying non-Euclidean regularity assumptions.

The design of sampling algorithms under general geometries (e.g.\ ``mirrored Langevin algorithms'') has been a topic of great recent interest, independently from applications in private optimization. Obtaining mixing guarantees under regularity assumptions naturally found in applications is a notoriously challenging problem in the recent algorithmic sampling literature \cite{HsiehKRC18, ZhangPFP20, AhnC21, Jiang21, LiTVW22}. For example, it has been observed both theoretically and empirically that without (potentially restrictive) relationships between regularity parameters, natural discretizations of the mirrored Langevin dynamics may not even have vanishing bias \cite{ZhangPFP20, Jiang21, LiTVW22}. Recently, \cite{lee2021structured} gave an alternative strategy (to discretizing Langevin dynamics) for designing sampling algorithms in the Euclidean case, used in \cite{GLL22} to obtain private algorithms for $\ell_2$-structured ERM and SCO problems (see Proposition~\ref{prop:sample_elltwo}). Our work suggests a natural non-Euclidean generalization of these sampling problems, which is useful to study from an algorithmic perspective. We are optimistic that a non-Euclidean variant of \cite{lee2021structured} may shed light on these mysteries and yield new efficient private algorithms. More generally (beyond the particular \cite{lee2021structured} framework), we state the direction of designing efficient samplers for densities of the form $\exp(-F_{\data} - \mu r)$ satisfying Assumption~\ref{assume:norm} as an important open research endeavor with implications for both sampling and private optimization, the latter of which this paper demonstrates.


\section{Preliminaries}
\label{sec:prelims}

\paragraph{General notation.} Throughout, $\tO$ hides logarithmic factors in problem parameters when clear from the context. For $n \in \N$, $[n]$ refers to the naturals $1 \le i \le n$. We use $\xset$ to denote a compact convex subset of $\R^d$. The standard ($\ell_2$) Euclidean norm is denoted $\norm{\cdot}_2$. We will be concerned with optimizing functions $f: \xset \to \R$, and $\normx{\cdot}$ will refer to a norm on $\xset$. The diameter of such a set is denoted $\diam_{\normx{\cdot}}(\xset) \defeq \max_{x, y \in \xset}\normx{x - y}$. 
We let $\Nor(\mu, \msig)$ be the Gaussian density of specified mean and covariance. We denote the convex hull of a set $S$ (when well-defined) by $\conv(S)$. When $a, b \in \R^d$, we abuse notation and let $[a, b]$ be the line segment between $a$ and $b$.

\paragraph{Norms.} For $p \ge 1$, we let $\norm{\cdot}_p$ applied to a vector-valued variable be the $\ell_p$ norm, namely $\norm{v}_p = (\sum_{i \in [d]} |v_i|^p)^{1/p}$ for $v \in \R^d$; the $\ell_\infty$ norm is the maximum absolute value. We will use the well-known inequality
\begin{equation}\label{eq:normequiv}\norm{v}_p \le \norm{v}_{q} \le d^{\frac 1 q - \frac 1 p}\norm{v}_p, \text{ for } v \in \R^d,\; q \le p.\end{equation}
Matrices will be denoted in boldface throughout, and $\norm{\cdot}_p$ applied to a matrix-valued variable $\mm$ is the Schatten-$p$ norm, i.e.\ the $\ell_p$ norm of the singular values of $\mm$.

\paragraph{Optimization.} In the following discussion, fix some $f: \xset \to \R$. We say $f$ is $G$-Lipschitz in $\normx{\cdot}$ if for all $x, x' \in \xset$, $|f(x) - f(x')| \le G\normx{x - x'}$. We say $f$ is $\mu$-strongly convex in $\normx{\cdot}$ if for all $x, x' \in \xset$ and $t \in [0, 1]$,
\[f\Par{tx + (1 - t)y} \le tf(x) + (1 - t) f(y) - \frac{\mu t(1-t)} 2 \normx{x - y}^2.\]

\paragraph{Probability.} For two densities $\pi, \pi'$, we define their total variation distance by $\tvd{\pi}{\pi'} \defeq \half \int |\pi(x) - \pi'(x)|dx$ and (when the Radon-Nikodym derivative exists) their KL divergence by $\dkl(\pi \| \pi') \defeq \int \pi(x) \log \frac{\pi(x)}{\pi'(x)} dx$. We define the $2$-Wasserstein distance by
\[W_2(\pi, \pi') = \inf_{\Gamma \in \Gamma(\pi, \pi')} \sqrt{\E_{(x, x') \sim \Gamma} \|x - x'\|_2^2},\]
where $\Gamma(\pi, \pi')$ is the set of couplings of $\pi$ and $\pi'$. We note $W_2$ satisfies the following inequality.
\begin{fact}\label{fact:w2lip}
Let $\textup{Lip}_2(f)$ be the Lipschitz constant in the $\ell_2$ norm of a function $f$. Then, for densities $\pi, \pi'$ supported on $\xset$, 
\[W_2(\pi, \pi') \ge \sup_{\textup{Lip}_2(f) \le 1} \int_{\xset} f(x) (\pi(x) - \pi'(x)) dx.\]
\end{fact}
\begin{proof}
This follows from the dual characterization of the $1$-Wasserstein distance (which shows $\sup_{\textup{Lip}(f) \le 1} \int_{\xset} f(x) (\pi(x) - \pi'(x)) dx = \inf_{\Gamma \in \Gamma(\pi, \pi')} \E_{(x, x') \sim \Gamma} \|x - x'\|_2$), and convexity of the square.
\end{proof}

We use $\propto$ to indicate proportionality, e.g.\ if $\pi$ is a density and we write $\pi \propto \exp(-f)$, we mean $\pi(x) = \frac{\exp(-f)}{Z}$ where $Z \defeq \int \exp(-f(x)) dx$ and the integration is over the support of $\pi$. 

We say that a measure $\pi$ on $\R^d$ is logconcave if for any $\lam \in (0, 1)$ and compact $A, B \subset \R^d$,
\[\pi(\lam A + (1 - \lam)B) \ge \pi(A)^{\lam} \pi(B)^{1 - \lam}.\]
We have the following equivalent characterization of logconcave measures.

\begin{proposition}[\cite{Bor75}]
	\label{thm:character}
	Let $\pi$ be a measure on $\R^d$.
	Let $E$ be the least affine subspace containing the support of $\pi$, and let $m_E$ be the Lebesgue measure in $E$.
	Then $\pi$ is logconcave if and only if $d \pi =f dm_E$, $f$ is nonnegative and locally integrable, and $-\log f: E \to \R \cup \{+\infty\}$ is convex.
\end{proposition}

In particular, Proposition~\ref{thm:character} shows that the measure of any subspace of $E$ according to $\pi$ is zero. If in the characterization of \cite{Bor75} the function $-\log f$ is affine, we say $\pi$ is logaffine. Following \cite{Bor75}, we analogously define strong logconcavity with respect to a norm.

\begin{definition}[strong logconcavity]
Let $\pi$ be a measure on $\R^d$. Let $E$ be the least affine subspace containing the support of $\pi$, and let $m_E$ be the Lebesgue measure in $E$. We say $\pi$ is $\mu$-strongly logconcave with respect to $\normx{\cdot}$ if $d\pi = fdm_E$, $f$ is nonnegative and locally integrable, and $-\log f: E \to \R \cup \{+\infty\}$ is $\mu$-strongly convex in $\normx{\cdot}$.
\end{definition}

\paragraph{Privacy.} Throughout, $\mech$ denotes a mechanism, and $\data$ denotes a dataset. We say $\data$ and $\data'$ are neighboring if they differ in one entry. We say a mechanism $\mech$ satisfies $(\eps, \delta)$-differential privacy if it has output space $\Omega$ and for any neighboring $\data, \data'$, 
\[\sup_{S \subseteq \Omega} \Pr[\mech(\data) \in S] - \exp(\eps) \Pr[\mech(\data') \in S] \le \delta.\]
We define the privacy curve of two random variables $X, Y$ supported on $\Omega$ by
\[\delta(X \| Y)(\eps) \defeq \sup_{S \subseteq \Omega} \Pr[Y \in S] - \exp(\eps) \Pr[X \in S].\]
We say $\mech$ has a privacy curve $\delta: \R_{\ge 0} \to [0, 1]$ if for all neighboring $\data$, $\data'$, $\delta(\mech(\data) \| \mech(\data')) \le \delta(\eps)$. For any $\eps \in \R_{\ge 0}$, it is clear that such a $\mech$ is $(\eps, \delta(\eps))$-differentially private. We will frequently compare to the privacy curve of a Gaussian, so we recall the following bound from prior work.

\begin{fact}[Gaussian privacy curve, Lemma 6.3, \cite{GLL22}]\label{fact:gprivacy}
Let $\delta \in (0, \half)$ and $\eps > 0$. For any $|t| \le \sqrt{2\log \frac 1 {2\delta} + 2\eps} - \sqrt{2\log \frac 1 {2\delta}} \le \frac{\eps}{\sqrt{2\log \frac 1 {2\delta}}}$, $\deltacurve{\Nor(0, 1)}{\Nor(t, 1)}(\eps) \le \delta$.
\end{fact}

We will use Fact~\ref{fact:gprivacy} after deriving our Gaussian differential privacy guarantees \cite{DongRS21} for strongly logconcave densities in Theorem~\ref{thm:privacy_curve_norm_X}.

\section{Gaussian differential privacy in general norms}
\label{sec:privacy}

In this section, we give a generalization of Theorem 4.1 of \cite{GLL22}, which demonstrates that a regularized exponential mechanism for (Euclidean norm) Lipschitz losses achieves privacy guarantees comparable to an analogous instance of the Gaussian mechanism. The proof from \cite{GLL22} was specialized to the Euclidean setup; to show our more general result, we draw upon the localization technique from convex geometry \cite{lovasz1993random, kannan1995isoperimetric}. We provide the relevant localization tools we will use in Section~\ref{ssec:localization}, and prove our Gaussian differential privacy result in Section~\ref{ssec:gdp}.

\subsection{Localization}\label{ssec:localization}

We recall the localization lemma from \cite{FG04}. We remark that the statement in \cite{FG04} is more refined than our statement (in that \cite{FG04} gives a complete characterization of extreme points, whereas we give a superset), but the following form of the \cite{FG04} result suffices for our purposes.

\begin{proposition}[Theorem 1, \cite{FG04}]
	\label{thm:localization}
	Let $\xset \subset \R^d$ be compact and convex, and let $f: \xset \to \R$ be upper semi-continuous. 
	Let $\set(f)$ be the set of logconcave densities $\nu$ supported in $\xset$ satisfying $\int_{\xset} fd \nu \geq 0$.
	The set of extreme points of $\conv(\set(f))$ satisfies one of the following.
	\begin{itemize}
		\item $\nu$ is a Dirac measure at $x \in \xset$ such that $f(x) \ge 0$.

		\item $\nu$ is logaffine and supported on $[a,b]\subset \xset$ such that $\int_{[a,b]} f d\nu= 0$. 
	\end{itemize}
\end{proposition}

We next derive several extensions of Proposition~\ref{thm:localization}.

\begin{lemma}[Strongly logconcave localization]
	\label{lem:localization_stronglylogconcave}
	Let $\xset \subset \R^d$ be compact and convex, let $\beta: \xset \to \R_{> 0}$ be continuous, and let $f: \xset \to \R$ be upper semi-continuous.
	Let $\set_{\mu, \beta}(f)$ be the set of probability densities $\pi$ such that $\pi$ is $\mu$-strongly logconcave with respect to $\normx{\cdot}$ and supported in $\xset$, such that $\pi' \propto \beta \pi$ is also $\mu$-strongly logconcave, and $\int_{\xset} fd \pi\geq 0$.
	The set of extreme points of $\conv(\set_{\mu, \beta}(f))$ satisfy one of the following.
	\begin{itemize}
		\item $\pi$ is a Dirac measure at $x \in \xset$ such that $f(x) \ge 0$.
		\item $\pi$ is supported on $[a, b] \subset \xset$.
	\end{itemize}
\end{lemma}
\begin{proof}
Clearly, Dirac measures at $x$ with $f(x) \ge 0$ are extreme points, so it suffices to consider other extreme points. Given any extreme point $\pi$ which is not a Dirac measure, we prove the least affine subspace containing the support of $\pi$ has dimension one, i.e.\ denoting the least affine subspace containing the support of $\pi$ by $S$, we prove $\dim S=1$.
	
Assume for the sake of contradiction that $\dim S\geq 2$. There exists $x_0$ in the relative interior of the support of $\pi$ and a two-dimensional subspace $E \subset \R^d$ such that $x_0+E\subseteq S$.
	Let $S_1(E)$ be the unit circle in $E$, and for any $u\in S_1(E)$ denote $H_u=\{x\in S:\langle x-x_0,u\rangle=0\},H_u^+=\{x\in S:\langle x-x_0,u\rangle\geq 0\}$ and $H_u^-=\{x\in S:\langle x-x_0,u\rangle\leq 0\}$. Finally, define $\phi: S_1(E) \to \R$ by $\phi(u) \defeq \int_{H_u^+} fd\pi - \half (\int fd\pi)$, such that $\phi(u) = 0 \implies  \int_{H_u^+} fd\pi = \half \int fd\pi \ge 0$. 
	
	By Proposition~\ref{thm:character}, we know $\pi(H_u)=0$.
	Moreover, $\phi$ is continuous since every hyperplane $H_u$ has $\pi(H_u) = 0$. Since $\phi(u) = -\phi(-u)$, by the intermediate value theorem there exists $u_0 \in S_1(E)$ such that $\phi(u_0) = 0$. We can hence rewrite $\pi$ as a convex combination of its restrictions to $H_{u_0}^+$ and $H_{u_0}^-$, both of which are $\mu$-strongly logconcave, and whose (renormalized) multiplications by $\beta$ are also $\mu$-strongly logconcave. Since $\phi(u_0) = 0$ both of these restrictions belong to $\set_{\mu, \beta}(f)$, contradicting extremality of $\pi$.
\end{proof}

We briefly remark that the proof technique used in Lemma~\ref{lem:localization_stronglylogconcave} is quite general, and the only property we used about $\set_{\mu, \beta}$ is that it is a subset of logconcave densities, and it is closed under restrictions to convex subsets. Similar arguments hold for other density families with these properties. Further, we note that restrictions to compact sets are upper semi-continuous; it is straightforward to verify our applications of Lemma~\ref{lem:localization_stronglylogconcave} satisfy the upper semi-continuity assumption.

We prove the following two technical lemmas using Lemma~\ref{lem:localization_stronglylogconcave}.
\begin{lemma}
	\label{lm:reducet_onedim_privacy}
	Following notation of Lemma~\ref{lem:localization_stronglylogconcave}, fix a continuous function $\alpha:\xset \to \R$ and a subset $S\subset \xset$. For any probability density $\pi$ on $\xset$, define the renormalized density $\tpi \propto e^{-\alpha} \pi$. Finally, let
	$$g(\pi) \defeq \Pr_{x\sim \Tilde{\pi}}[x\in S]-e^\eps\Pr_{x\sim \pi}[x\in S].$$
	Then
	$\max_{\pi \in \set_{\mu, \beta}} g(\pi) = \max_{\pi \in \set^*_{\mu, \beta}} g(\pi)$ where 
	$\set_{\mu, \beta}^*$ is the subset of densities $\pi \in \set_{\mu, \beta}$ satisfying one of the following.
	\begin{itemize}
		\item $\pi$ is a Dirac measure at $x \in \xset$.
		\item $\pi$ is supported on $[a, b] \subset \xset$.
	\end{itemize}
\end{lemma}

\begin{proof}
	Let $\set_{\mu, \beta}(f) \subseteq \set_{\mu, \beta}$ be the set of $\pi \in \set_{\mu, \beta}$ such that $\int fd\pi \ge 0$. We have
	\begin{align*}
		\max_{\pi \in \set_{\mu, \beta}} g(\pi)&= \max_{\pi \in \set_{\mu, \beta}} \int_{x\in S} d\Tilde{\pi}(x) - e^\eps \int_{x\in S} d\pi(x)\\
		&= \max_{\pi\in \set_{\mu, \beta}} \frac{\int_{x\in S}e^{-\alpha(x)}d\pi(x)}{\int_{x\in\xset}e^{-\alpha(x)}d \pi(x)} -e^{\epsilon}\int_{x\in S}d\pi(x)\\
		&= \max_{\pi\in \set_{\mu, \beta}} \quad \max_{C \ge \int_{x\in \xset} e^{-\alpha(x)}d\pi(x)}  \frac{\int_{x\in S}e^{-\alpha(x)}d\pi(x)}{C}-e^{\epsilon}\int_{x\in S}d\pi(x)\\
		&= \max_C \max_{\pi\in \set_{\mu, \beta}(C-e^{-\alpha})}  \int_{x\in \xset}\lp \frac{e^{-\alpha(x)}}{C}-e^{\epsilon}\rp \Ind_S(x) d\pi(x) \\
		&= \max_C \max_{\pi\in \set_{\mu, \beta}(C-e^{-\alpha})^*}  \int_{x\in \xset}\lp \frac{e^{-\alpha(x)}}{C}-e^{\epsilon}\rp \Ind_S(x) d\pi(x),
	\end{align*}
	where $\set_{\mu, \beta}(C-e^{-\alpha})^*$ is the (super)set of extreme points of $\set_{\mu, \beta}(C-e^{-\alpha})$ given by the strongly logconcave localization lemma (Lemma~\ref{lem:localization_stronglylogconcave}). These candidate extreme points are Dirac measures at $x$ such that $C \ge e^{-\alpha(x)}$, or are supported in $[a, b] \subset \xset$. Hence, $\set_{\mu, \beta}(C - e^{-\alpha})^* \subseteq \set_{\mu, \beta}^*$, and we conclude
	\begin{align}
		\label{eq:local_maxg}
		\max_{\pi \in \set_{\mu, \beta}} g(\pi) &= \max_C \max_{\pi\in \set_{\mu, \beta}(C-e^{-\alpha})^*}  \int_{x\in \xset}\lp \frac{e^{-\alpha(x)}}{C}-e^{\epsilon}\rp \Ind_S(x) d\pi(x) \\
		&\le \max_C \max_{\pi \in \set_{\mu, \beta}(C-e^{-\alpha})^*}  \int_{x\in \xset}\lp \frac{e^{-\alpha(x)}}{\int_{x\in \xset}e^{-\alpha(x)}d\pi(x)}-e^{\epsilon}\rp \Ind_S(x) d\pi(x) \\
		&\le \max_{\pi\in \set_{\mu, \beta}^*}  \int_{x\in \xset}\lp \frac{e^{-\alpha(x)}}{\int_{x\in \xset}e^{-\alpha(x)}d\pi(x)}-e^{\epsilon}\rp \Ind_S(x) d\pi(x) = \max_{\pi \in \set_{\mu, \beta}^*} g(\pi).
	\end{align}
The first inequality used that $C \ge \int_{x \in \xset} e^{-\alpha(x)}d\pi(x)$ for $\pi \in \set_{\mu, \beta}(C - e^{-\alpha})^*$, and the second used that $\set_{\mu, \beta}(C - e^{-\alpha})^* \subseteq \set_{\mu, \beta}^*$ for any $C$. Since $\set_{\mu, \beta}^* \subseteq \set_{\mu, \beta}$, we have the claim.
\end{proof}

\begin{lemma}
	\label{lm:reduceto_onedim}
	Following notation of Lemma~\ref{lem:localization_stronglylogconcave}, fix continuous function $\alpha: \xset \to \R$ and upper semi-continuous function $f: \xset \to \R$. For any probability density $\pi$ on $\xset$, define $\tpi \propto e^{-\alpha} \pi$ to be a renormalized density on $\xset$. Finally, let 
	\begin{align*}
		g(\pi) \defeq \int_{x \in \xset}f(x)d (\pi-\Tilde{\pi})(x).
	\end{align*}
	Then
	$\max_{\pi \in \set_{\mu, \beta}} g(\pi) = \max_{\pi \in \set_{\mu, \beta}^*} g(\pi)$ where 
	$\set_{\mu, \beta}^*$ is the subset of densities $\pi \in \set_{\mu, \beta}$ satisfying one of the following.
	\begin{itemize}
		\item $\pi$ is a Dirac measure at $x \in \xset$.
		\item $\pi$ is supported on $[a, b] \subset \xset$.
	\end{itemize}
\end{lemma}

\begin{proof}
	We follow the notation from Lemma~\ref{lm:reducet_onedim_privacy}. If $\pi$ is a Dirac measure, $g(\pi)=0$, so we only need to consider the case when $\max_{\pi\in \set_{\mu, \beta}}g(\pi)>0$. We have
	\begin{align*}
		\max_{\pi \in \set_{\mu, \beta}}g(\pi) &= \max_{\pi \in \set_{\mu, \beta}} \int_{x \in \xset} f(x)\Par{1-\frac{e^{-\alpha(x)}}{\int_{x \in \xset}e^{-\alpha(x)}d \pi(x)}}d \pi(x)\\
		&= \max_{\pi \in \set_{\mu, \beta}} \quad \max_{C\le \int_{x\in \xset}e^{-\alpha(x)}d\pi(x)}\int_{x \in \xset}\Par{f(x)-\frac{e^{-\alpha(x)} f(x)}{C}}d \pi(x)\\
		&= \max_{C} \max_{\pi\in \set_{\mu, \beta}(e^{-\alpha} - C)}\int_{x \in \xset}\Par{f(x)-\frac{e^{-\alpha(x)} f(x)}{C}}d \pi(x) \\
		&=\max_{C} \max_{\pi\in \set_{\mu, \beta}(e^{-\alpha} - C)^*}\int_{x \in \xset}\Par{f(x)-\frac{e^{-\alpha(x)} f(x)}{C}}d \pi(x).
	\end{align*}
	The remainder of the proof is analogous to Lemma~\ref{lm:reducet_onedim_privacy}.
\end{proof}

\subsection{Gaussian differential privacy}\label{ssec:gdp}

Using an instantiation of the localization lemma, we prove Gaussian differential privacy in general norms by first reducing to one dimension and then using the result of \cite{GLL22} to handle the one-dimensional case. Gaussian differential privacy was introduced by \cite{DongRS21} and is a useful tool to compare privacy curves. We first recall the ($\ell_2$) Gaussian differential privacy result of \cite{GLL22}.

\begin{proposition}[Theorem 4.1, \cite{GLL22}]\label{prop:gll}
Let $\xset \subset \R^d$ be compact and convex, let $F, \tF: \xset \to \R$ be $\mu$-strongly convex in $\norm{\cdot}_2$, and let $P \propto \exp(-F)$ and $Q \propto \exp(-\tF)$. Suppose $\tF - F$ is $G$-Lipschitz in $\norm{\cdot}_2$. For all $\eps \in \R_{\ge 0}$,
\[\deltacurve{P}{Q}(\eps) \le \deltacurve{\Nor(0, 1)}{\Nor\Par{\frac G {\sqrt \mu}, 1}}(\eps).\]
\end{proposition}

We next give a simple comparison result between norms.

\begin{lemma}\label{lem:normxtol2}
For $f: \xset \to \R$, fix $a, b \in \xset$, and let $\tf: [a, b] \to \R$ be the restriction of $f$ to $[a, b]$.
\begin{enumerate}
	\item If $f$ is $G$-Lipschitz in $\normx{\cdot}$, $\tf$ is $G \cdot \frac{\normx{b - a}}{\norm{b - a}_2}$-Lipschitz in $\norm{\cdot}_2$.
	\item If $f$ is $\mu$-strongly convex in $\normx{\cdot}$, $\tf$ is $\mu \cdot \frac{\normx{b - a}^2}{\norm{b - a}_2^2}$-strongly convex in $\norm{\cdot}_2$.
\end{enumerate}
\end{lemma}
\begin{proof}
To see the first claim, let $c = a + r(b - a)$ and $d = a + s(b - a)$ for $s, r \in [0, 1]$. We have by Lipschitzness of $f$ that
\begin{align*}
\Abs{\tf(d) - \tf(c)} \le G\Abs{s - r}\normx{b - a} = \Par{G \cdot \frac{\normx{b - a}}{\norm{b - a}_2}} \cdot \norm{d - c}_2.
\end{align*}
Similarly, to see the second claim, by strong convexity of $f$,
\begin{align*}\tf\Par{tc + (1 - t)d} &\le t\tf(c) + (1 - t)\tf(d) - \frac{\mu t(1 - t)}{2} \normx{c - d}^2 \\
&= t\tf(c) + (1 - t)\tf(d) - \frac{\mu t(1 - t)(r - s)^2}{2} \normx{a - b}^2 \\
&= t\tf(c) + (1 - t)\tf(d) - \Par{\mu \cdot \frac{\normx{b - a}^2}{\norm{b - a}_2^2}}\Par{\frac{t(1 - t)}{2} \norm{d - c}_2^2}.
\end{align*}
\end{proof}

We now present our main result on Gaussian differential privacy with respect to arbitrary norms.

\restategdp*
\begin{proof}
Throughout this proof, fix some $\alpha$ which is $G$-Lipschitz in $\normx{\cdot}$ by assumption. We first claim that amongst all $\mu$-strongly convex (in $\normx{\cdot}$) functions $F: \xset \to \R$ such that $F + \alpha$ is also $\mu$-strongly convex, defining $P \propto \exp(-F)$ and $Q \propto \exp(-(F + \alpha))$, some $F$ maximizing $\deltacurve{P}{Q}(\eps)$ is either a Dirac measure or supported on $[a, b] \subset \xset$. We will prove this by contradiction.

Suppose otherwise, and let $F$ be a $\mu$-strongly convex function that maximizes $\deltacurve{P}{Q}(\eps)$ defined above. Define $P \propto \exp(-F)$ and $Q \propto \exp(-(F + \alpha))$. Let $S^* \subseteq \xset$ be the set achieving 
\[\deltacurve{P}{Q}(\eps) = \Pr_{X \sim P}[X \in S^*] - \exp(\eps) \Pr_{X \sim Q}[X \in S^*].\]
By Lemma~\ref{lm:reducet_onedim_privacy}, there is another $\mu$-strongly logconcave $\pi$ where the renormalized density $\propto \pi \exp(-\alpha)$ is also $\mu$-strongly logconcave, such that (following notation of Lemma~\ref{lm:reducet_onedim_privacy}) $g(\pi) \ge g(P)$, where $\pi$ is either a Dirac or supported on $[a, b]$. We conclude that $\deltacurve{P}{Q}(\eps) \le \deltacurve{\pi}{\pi \exp(-\alpha)}(\eps)$ (since the maximizing set for $\pi$ is at least as good as $S^*$), a contradiction.

It hence suffices to prove the theorem statement for $F, \tF$, which are supported on some $[a, b] \in \xset$. By Lemma~\ref{lem:normxtol2}, we have that $\tF - F$ is $G \cdot \frac{\normx{b - a}}{\norm{b - a}_2}$-Lipschitz in $\norm{\cdot}_2$ and $F, \tF$ are $\mu \cdot \frac{\normx{b - a}^2}{\norm{b - a}_2^2}$-strongly convex in $\norm{\cdot}_2$. We conclude by Proposition~\ref{prop:gll} which shows
\begin{align*}\deltacurve{P}{Q}(\eps) &\le \deltacurve{\Nor(0, 1)}{\Nor\Par{\frac G {\sqrt \mu} \cdot \frac{\normx{b - a}}{\norm{b - a}_2} \cdot \frac{\norm{b - a}_2}{\normx{b - a}}, 1}}(\eps) \\
	&= \deltacurve{\Nor(0, 1)}{\Nor\Par{\frac G {\sqrt \mu}, 1}}(\eps).\end{align*}
\end{proof}

Our proof strategy is a reduction to an application of Proposition~\ref{prop:gll} in one dimension. It is an interesting open question to obtain a simpler direct proof of Proposition~\ref{prop:gll} in the one-dimensional setting (without using the machinery of \cite{GLL22}), which is tight up to constant factors.

\section{Private ERM and SCO in general norms}
\label{sec:apps}

In this section, we derive our results for private ERM and SCO in general norms. We will state our results for private ERM (Section~\ref{ssec:erm}) and SCO (Section~\ref{ssec:sco}) with respect to an arbitrary compact convex subset $\xset$ of a $d$-dimensional normed space, satisfying Assumption~\ref{assume:norm}. We then use this to derive guarantees for a variety of settings of import in Section~\ref{ssec:apps}.

\subsection{Private ERM under Assumption~\ref{assume:norm}}\label{ssec:erm}

To develop our private ERM algorithms, we recall the following risk guarantee from \cite{KlerkL18} of sampling from Gibbs distributions (improving upon \cite{KalaiV06, BassilyST14}).

\begin{proposition}[\cite{KlerkL18}, Corollary 1]\label{prop:gibbsrisk}
Let $\xset \subset \R^d$ be compact and convex, let $F: \xset \to \R$ be convex, and let $k > 0$. If $\nu \propto \exp(-kF)$,
\[\E_{x \sim \nu}[F(x)] \le \min_{x \in \xset} F(x) + \frac d k.\]
\end{proposition}

We conclude by a simple combination of Proposition~\ref{prop:gibbsrisk} (providing a risk guarantee) and Theorem~\ref{thm:privacy_curve_norm_X} (providing a privacy guarantee), which yields our main result on private ERM.

\begin{theorem}[Private ERM]\label{thm:erm}
Under Assumption~\ref{assume:norm} and following notation \eqref{eq:Fdef}, drawing a sample $x$ from the density $\nu \propto \exp(-k(F_{\data} + \mu r))$ for
\[k = \frac{\sqrt{d} n\eps}{G\sqrt{2\Theta \log \frac 1 {2\delta}}},\; \mu = \frac{G\sqrt{2d \log \frac 1 {2\delta}}}{\sqrt{\Theta} n\eps},\]
is $(\eps, \delta)$-differentially private, and produces $x$ such that
\[\E_{x \sim \nu}[F_{\data}(x)] - \min_{x \in \xset} F_{\data}(x) \le G\sqrt{\Theta} \cdot \frac{\sqrt{8d\log \frac 1 {2\delta}}}{n\eps}.\]
\end{theorem}
\begin{proof}
Let $F_{\data'}$ be the realization of \eqref{eq:Fdef} when $\data$ is replaced with a neighboring dataset $\data'$ which agrees in all entries except some sample $s'_i \neq s_i$. By Assumption~\ref{assume:norm}, we have $k(F_{\data} - F_{\data'})$ is $\frac{kG}{n}$-Lipschitz, and both $k(F_{\data} + \mu r)$ and $k(F_{\data'} + \mu r)$ are $k\mu$-strongly convex (all with respect to $\normx{\cdot}$). Hence, combining Theorem~\ref{thm:privacy_curve_norm_X} and Fact~\ref{fact:gprivacy} shows the mechanism is $(\eps, \delta)$-differentially private, since
\begin{equation}\label{eq:kmudef}\mu = \frac{2G^2 k \log \frac 1 {2\delta}}{n^2 \eps^2} \implies \frac{G\sqrt k}{n\sqrt \mu} \le \frac \eps {\sqrt{2 \log \frac 1 {2\delta}}}.\end{equation}
Let $x^\star_{\data} \defeq \argmin_{x \in \xset} F_{\data}(x)$. We obtain the risk guarantee by the calculation (see Proposition~\ref{prop:gibbsrisk})
\begin{align*}
\E_{x \sim \nu}[F_{\data}(x)] &\le F_{\data}(x^\star_{\data}) + \Par{\mu r(x^\star_{\data}) - \E_{x \sim \nu} \mu r(x)} + \frac d k \\
&\le F_{\data}(x^\star_{\data}) + \mu \Theta + \frac d k
\end{align*}
and plugging in our choices of $\mu$ and $k$.
\end{proof}

\subsection{Private SCO under Assumption~\ref{assume:norm}}\label{ssec:sco}

We first give a generic comparison result between population risk and empirical risk under Assumption~\ref{assume:norm}. To do so, we use two helper results from prior work. The first was derived in \cite{GLL22} by combining a transportation inequality and a log-Sobolev inequality (see e.g.\ \cite{otto2000generalization}).

\begin{proposition}[\cite{GLL22}, Theorem 6.7, Lemma 6.8]\label{prop:w2bound}
Let $\xset \subseteq \R^d$ be compact and convex, let $F, \tF: \xset \to \R$ be $\mu$-strongly convex in $\norm{\cdot}_2$, and let $P \propto \exp(-F)$ and $Q \propto \exp(-\tF)$. Suppose $\tF- F$ is $H$-Lipschitz in $\norm{\cdot}_2$. Then, $W_2(P, Q) \le \frac H \mu$.
\end{proposition}

\begin{corollary}\label{cor:liplocalization}
Let $\xset \subset \R^d$ be compact and convex, and let $\alpha, f: \xset \to \R$ be $H$-Lipschitz and $G$-Lipschitz respectively in $\normx{\cdot}$. Let $\set_{\mu, \exp(-\alpha)}$ be the set of densities $\pi$ over $\xset$ such that $\pi$ is $\mu$-strongly logconcave with respect to $\normx{\cdot}$, and $\tpi \propto \pi \exp(-\alpha)$ is also $\mu$-strongly logconcave. For any $\pi \in \set_{\mu, \exp(-\alpha)}$ define $g(\pi) \defeq \int_{\xset} f(x) d\Par{\pi - \tpi}(x)$ where $\tpi \propto \pi \exp(-\alpha)$. Then, $g(\pi) \le \frac{GH}{\mu}$.
\end{corollary}
\begin{proof}
By Lemma~\ref{lm:reduceto_onedim} (and following its notation), it suffices to show $g(\pi) \le \frac{GH}{\mu}$ for all $\pi \in \set^*_{\mu, \exp(-\alpha)}$. Clearly this is true for a Dirac measure $\pi$ as then $g(\pi) = 0$, so consider the other case where $\pi$ is supported on $[a, b]$, such that $\pi \propto \exp(-F)$ and $F$ is $\mu$-strongly convex in $\normx{\cdot}$. Further, define $\tF = F + \alpha$, so that $\tF$ is also $\mu$-strongly convex and supported on $[a, b]$.

By Lemma~\ref{lem:normxtol2}, restricting to $[a,b]$, $F$ and $\tF$ are $\mu \cdot \frac{\normx{b - a}^2}{\norm{b - a}_2^2}$-strongly convex in $\norm{\cdot}_2$, $F - \tF$ is $H \cdot \frac{\normx{b - a}}{\norm{b - a}_2}$-Lipschitz in $\ell_2$ and $f$ is $\frac{\normx{b-a}}{\norm{b-a}_2}$-Lipschitz in $\norm{\cdot}_2$. Hence, where the inequalities are by Fact~\ref{fact:w2lip} and Proposition~\ref{prop:w2bound} respectively,
\[g(\pi) = \int_{\xset} f(x) d(\pi - \tpi)(x) \le G W_2(\pi, \tpi) \le \frac{GH}{\mu}.\]
\end{proof}

The second relates the population risk to the empirical risk on an independent sample.

\begin{proposition}[Lemma 7, \cite{BousquetE02}]\label{prop:poprisk}
Suppose $\data = \{s_i\}_{i \in [n]}$ is drawn independently from $\calP$, let $s \sim \calP$ be drawn independently from $\data$, and let $\data' \defeq \{s\} \cup \{s_i\}_{i \in [n] \setminus \{1\}}$ be $\data$ where $s_1$ is swapped with $s$. Then, for any symmetric\footnote{Here, a symmetric mechanism is one which only depends on the set of inputs rather than their order.} mechanism $\mech: \textup{supp}(\calP)^n \to \R^d$, 
\[\E\Brack{\Fpop(\mech(\data)) - F_{\data}(\mech(\data))} = \E\Brack{f(\mech(\data); s) - f(\mech(\data'); s)},\]
where expectations are over $\mech$ and the randomness used in producing $\data$ and $s$.
\end{proposition}

By applying Corollary~\ref{cor:liplocalization} and Proposition~\ref{prop:poprisk} (which bound the generalization error of our mechanism), we provide the following extension of Theorem~\ref{thm:erm}, our main result on private SCO.

\begin{theorem}[Private SCO]\label{thm:sco}
Under Assumption~\ref{assume:norm} and following notation \eqref{eq:Fdef}, drawing a sample $x$ from the density $\nu \propto \exp(-k(F_{\data} + \mu r))$ for
\[k = \sqrt{\frac{d + C_2}{C_1}},\; \mu = \frac{2G^2 k \log \frac 1 {2\delta}}{n^2 \eps^2},\; C_1 \defeq \frac{2G^2\Theta \log \frac 1 {2\delta}}{n^2\eps^2},\; C_2 \defeq \frac{n\eps^2}{2 \log \frac 1 {2\delta}},\]
is $(\eps, \delta)$-differentially private, and produces $x$ such that 
\[\E_{\data \sim \calP^n, x \sim \nu}\Brack{\Fpop(x)} - \min_{x \in \xset} \Fpop(x) \le G\sqrt{\Theta} \cdot \Par{\frac{\sqrt{8d\log \frac 1 {2\delta}}}{n\eps} + \sqrt{\frac 8 n}}.\]
\end{theorem}
\begin{proof}
For the given choice of $k, \mu$, the privacy proof follows identically to Theorem~\ref{thm:erm}, so we focus on the risk proof. We follow the notation of Proposition~\ref{prop:poprisk} and let $s \sim \calP$ independently from $\data$. By exchanging the expectation and minimum and using that $\E_{\data \sim \calP^n} F_{\data} = \Fpop$,
\begin{align*}\E_{\data \sim \calP^n, x \sim \nu}\Brack{\Fpop(x)} - \min_{x \sim \xset} \Fpop(x) &\le \E_{\data \sim \calP^n}\Brack{\E_{x \sim \nu}\Brack{\Fpop(x)} - \min_{x \in \xset}\Brack{F_{\data}(x)}} \\
	&\le \E_{\data \sim \calP^n}\Brack{\E_{x \sim \nu}\Brack{\Fpop(x) - F_{\data}(x)}} \\
	&+ \E_{\data \sim \calP^n}\Brack{\E_{x \sim \nu}\Brack{F_{\data}(x)} - \min_{x \in \xset}\Brack{F_{\data}(x)}} \\
	&\le \E_{\data \sim \calP^n}\Brack{\E_{x \sim \nu}\Brack{\Fpop(x) - F_{\data}(x)}} + \mu\Theta + \frac d k,
\end{align*}
where we bounded the empirical risk in the proof of Theorem~\ref{thm:erm}. Next, let $\nu'$ be the density $\propto \exp(-k(F_{\data'} + \mu r))$. Our mechanism is symmetric, and hence by Proposition~\ref{prop:poprisk},
\[
\E\Brack{\Fpop(x) - F_{\data}(x)} = \E\Brack{\E_{x \sim \nu}\Brack{f(x; s)} - \E_{x \sim \nu'}\Brack{f(x; s)}}
\]
where the outer expectations are over the randomness of drawing $\data, s$. Finally, for any fixed realization of $\data, s$, the densities $\nu, \nu'$ satisfy the assumption of Corollary~\ref{cor:liplocalization} with $H = \frac G n$, and $f(\cdot; s)$ is $G$-Lipschitz, so Corollary~\ref{cor:liplocalization} shows that
\[\E_{x \sim \nu}\Brack{f(x; s)} - \E_{x \sim \nu'}\Brack{f(x; s)} \le \frac{G^2}{n\mu}.\]
Combining the above three displays bounds the population risk by
\begin{align*}\E_{\data \sim \calP^n, x \sim \nu}\Brack{\Fpop(x)} - \min_{x \in \xset} \Fpop(x) &\le \frac{G^2}{n\mu} + \mu\Theta + \frac d k \\
&= C_1 k + \frac{C_2 + d}{k},
\end{align*}
for our given value of $\mu$. The conclusion follows by optimizing over $k$ yielding a risk of $2\sqrt{C_1(C_2 + d)}$, and using the scalar inequality $\sqrt{a + b} \le \sqrt{a} + \sqrt{b}$ for nonnegative $a, b$.
\end{proof}

\subsection{Applications}\label{ssec:apps}

To derive our private optimization algorithms for $\ell_p$-norm and Schatten-$p$ norm geometries, we recall the following results on the existence of bounded strongly convex regularizers.

\begin{proposition}[\cite{BallCL94}]\label{prop:strongconvex}
For $1 < p \le 2$, letting $\norm{\cdot}_p$ be the $\ell_p$ norm of a vector, $r(v) \defeq \frac{1}{2(p-1)} \|v\|_p^2$ is $1$-strongly convex in $\norm{\cdot}_p$. Similarly, for $1 < p \le 2$, letting $\norm{\cdot}_p$ be the Schatten-$p$ norm of a matrix, $r(\mm) \defeq \frac{1}{2(p - 1)} \|\mm\|_p^2$ is $1$-strongly convex in $\norm{\cdot}_p$.
\end{proposition}

We state a useful result on efficiently sampling from Lipschitz, strongly logconcave densities under value oracle access given by \cite{GLL22} (building upon the framework of \cite{lee2021structured}). We slightly specialize the result of \cite{GLL22} by giving a rephrasing sufficient for our purposes.
\begin{proposition}[\cite{GLL22}, Theorem 2.3]
	\label{prop:sample_elltwo}
	Let $\xset \subset \R^d$ be compact and convex with $\diam_{\norm{\cdot}_2}(\xset) \le D$. Let $\data = \{s_i\}_{i \in [n]}$ and let $\tF_{\data}(x) = \frac 1 n \sum_{i \in [n]} f(x; s_i) + \psi(x)$ such that all $f(\cdot; s_i): \xset \to \R$ are $G$-Lipschitz in $\norm{\cdot}_2$ and convex, and $\psi(x): \xset \to \R$ is $\mu$-strongly convex in $\norm{\cdot}_2$. For $\delta \in (0, \half)$, we can generate a sample within total variation $\delta$ of the density $\propto \exp(-\tF_{\data})$ in $N$ value queries to some $f(\cdot; s_i)$ and samples from densities $\propto  \exp\Par{-\psi - \frac 1 {2\eta}\|\cdot - v\|_2^2}$ for some $\eta > 0$, $v \in \R^d$, where
	\begin{gather*}
	N = O\Par{\frac{G^2}{\mu}\log^2\Par{\frac {G^2(D^2 + \mu^{-1})d} \delta}}.
	\end{gather*}
\end{proposition}

\paragraph{$\ell_p$ norms.} We state our results on private convex optimization under $\ell_p$ geometry. As a preliminary, we combine norm equivalence bounds \eqref{eq:normequiv} and Proposition~\ref{prop:sample_elltwo} to give the following algorithmic result on sampling from a logconcave distribution under value oracle access under $\ell_p$ geometry. 
\begin{proposition}
	\label{prop:sample_ellp}
	Let $p \ge 1$ and let $\xset \subset \R^d$ be compact and convex with $\diam_{\norm{\cdot}_p}(\xset) \le D$. Let $\data = \{s_i\}_{i \in [n]}$ and let $\tF_{\data}(x) = \frac 1 n \sum_{i \in [n]} f(x; s_i) + \psi(x)$ such that all $f(\cdot; s_i): \xset \to \R$ are $G$-Lipschitz in $\norm{\cdot}_p$ and convex, and $\psi(x): \xset \to \R$ is $\mu$-strongly convex in $\norm{\cdot}_p$. For $\delta \in (0, \half)$, we can generate a sample within total variation $\delta$ of the density $\propto \exp(-\tF_{\data})$ in $N$ value queries to some $f(\cdot; s_i)$ and samples from densities $\propto  \exp\Par{-\psi - \frac 1 {2\eta}\|\cdot - v\|_2^2}$ for some $\eta > 0$, $v \in \R^d$, where
	\begin{align*}
	N &= 
	 O\Par{\frac{G^2d^{\frac 2 p - 1}}{\mu}\log^2\Par{\frac {G^2(D^2 + \mu^{-1})d} \delta}} \text{ if } p \in [1, 2],\\
	  N &= O\Par{\frac{G^2d^{1 - \frac 2 p}}{\mu}\log^2\Par{\frac {G^2(D^2 + \mu^{-1})d} \delta}} \text{ if } p \in [2, \infty).
	\end{align*}
\end{proposition} 
\begin{proof}
For $p \in [1, 2]$, note that each $f(\cdot; s_i)$ is $d^{\frac 1 p - \half} G$-Lipschitz in the $\ell_2$ norm by combining \eqref{eq:normequiv} and the definition of Lipschitzness. Moreover, because the $\ell_p$ norm is larger than the $\ell_2$ norm, $\psi$ remains $\mu$-strongly convex in the $\ell_2$ norm. The diameter $D$ is only affected by $\text{poly}(d)$ factors when converting norms, which is accounted for by the logarithmic term. Hence, the complexity bound follows by applying Proposition~\ref{prop:sample_elltwo} under this change of parameters. For the other case of $p \in [2, \infty)$, the Lipschitz bound is $G$, and the strong convexity bound is $d^{\frac 2 p - 1} \mu$ by a similar argument.
\end{proof}

In the following discussion, we primarily focus on the value oracle query complexity of our samplers. Generic results on logconcave sampling (see e.g.\ \cite{LovaszV07}, or more recent developments by \cite{JiaLLV21, Chen21, KlartagL22}) imply the samples from the densities $\propto \exp(-\psi - \frac 1 {2\eta} \|\cdot - v\|_2^2)$ can be performed in polynomial time, for all the $\psi$ that are relevant in our applications (which are all squared $\ell_p$ distances). We expect samplers which run in nearly-linear time (in $d$) may be designed for applications where $\xset$ is structured, such as an $\ell_p$ ball, but for brevity we omit this discussion.

\begin{corollary}\label{cor:pg1le2}
Let $1 < p \le 2$ be a constant, and let $\eps > 0$, $\delta \in (0, 1)$. Let $\xset \subset \R^d$ have $\diam_{\norm{\cdot}_p}(\xset) \le D$, and let $\Fpop = \E_{s \sim \calP}[f(\cdot; s)]$ where all $f(\cdot; s): \R^d \to \R$ are convex and $G$-Lipschitz in $\norm{\cdot}_p$. Finally, let $\data = \{s_i\}_{i \in [n]} \sim \calP^n$ independently and $F_{\data} \defeq \frac 1 n \sum_{i \in [n]} f(\cdot; s_i)$. 
\begin{enumerate}
	\item There is an $(\eps, \delta)$-differentially private algorithm $\mech$ which produces $x$ such that
	\begin{equation}\label{eq:pg1le2erm}\E_{\mech}\Brack{F_{\data}(x)} - \min_{x \in \xset} F_{\data}(x) \le 2GD \cdot \frac{\sqrt{d\log \frac 1 {2\delta}}}{n\eps \sqrt{p - 1}}\end{equation}
	using
	\begin{equation}\label{eq:g1le2N}O\Par{\frac{n^2\eps^2 d^{\frac 2 p - 1}}{\log \frac 1 \delta} \log^2\Par{\frac{GDdn\eps}{\delta}}} \text{ value queries to some } f(\cdot; s_i).\end{equation}
	\item There is an $(\eps, \delta)$-differentially private algorithm $\mech$ which produces $x$ such that
	\begin{equation}\label{eq:pg1le2sco}\E_{\data \sim \calP^n, \mech}\Brack{\Fpop(x)} - \min_{x \in \xset} \Fpop(x) \le 2GD \cdot \Par{\sqrt{\frac 1 {n (p - 1)}} + \frac{\sqrt{d\log \frac 1 {2\delta}}}{n\eps \sqrt{p - 1}}}.\end{equation}
	using
	\[O\Par{\frac{n^2\eps^2 d^{\frac 2 p - 1}}{\log \frac 1 \delta} \log^2\Par{\frac{GDdn\eps}{\delta}}} \text{ value queries to some } f(\cdot; s_i).\]
\end{enumerate}
\end{corollary}
\begin{proof}
We will parameterize Assumption~\ref{assume:norm} with the function $r(x) \defeq \frac 1 {2(p - 1)} \|x - x_0\|_p^2$, where $x_0 \in \xset$ is an arbitrary point, and strong convexity follows from Proposition~\ref{prop:strongconvex}. By assumption, we may set $\Theta = \frac 1 {2(p - 1)} D^2$. The conclusions follow by combining Theorem~\ref{thm:erm}, Theorem~\ref{thm:sco}. To obtain $(\eps, \delta)$-differential privacy, it suffices to run the mechanism with privacy level $\delta \gets \frac \delta 2$, run to total variation $\frac \delta 2$ using Proposition~\ref{prop:sample_ellp}, and take a union bound. For both ERM and SCO, note that our choices of $k$ and $\mu$ satisfy the relation \eqref{eq:kmudef}, namely $\frac {kG^2} \mu = O(n^2\eps^2 /\log \frac 1 \delta)$. Since both the Lipschitz and strong convexity parameters are scaled up by $k$ in our application of Proposition~\ref{prop:sample_ellp}, we have the leading-order term is $\frac{kG^2}{\mu}$ which yields the conclusion.
\end{proof}

For any $p$ such that $p - 1$ is bounded away from $0$, Corollary~\ref{cor:pg1le2} matches the information-theoretic lower bound of \cite{BassilyGN21} (and its subsequent sharpening by \cite{LiuL22}). When this is not the case, we use norm equivalence \eqref{eq:normequiv} to obtain a weaker bound.

\begin{corollary}\label{cor:pge1le2}
	Let $\eps > 0$, $\delta \in (0, 1)$. Let $\xset \subset \R^d$ have $\diam_{\norm{\cdot}_1}(\xset) \le D$, and let $\Fpop = \E_{s \sim \calP}[f(\cdot; s)]$ where all $f(\cdot; s): \R^d \to \R$ are convex and $G$-Lipschitz in $\norm{\cdot}_1$. Finally, let $\data = \{s_i\}_{i \in [n]} \sim \calP^n$ independently and $F_{\data} \defeq \frac 1 n \sum_{i \in [n]} f(\cdot; s_i)$. 
	\begin{enumerate}
		\item There is an $(\eps, \delta)$-differentially private algorithm $\mech$ which produces $x$ such that
		\begin{equation}\label{eq:pge1le2erm}\E_{\mech}\Brack{F_{\data}(x)} - \min_{x \in \xset} F_{\data}(x) \le 6GD \sqrt{\log d} \cdot \frac{\sqrt{d\log \frac 1 {2\delta}}}{n\eps }\end{equation}
		using
		\begin{equation}\label{eq:1N}O\Par{\frac{n^2\eps^2 d}{\log \frac 1 \delta} \log^2\Par{\frac{GDdn\eps}{\delta}}} \text{ value queries to some } f(\cdot; s_i).\end{equation}
		\item There is an $(\eps, \delta)$-differentially private algorithm  $\mech$ which produces $x$ such that
		\begin{equation}\label{eq:pge1le2sco}\E_{\data \sim \calP^n, \mech}\Brack{\Fpop(x)} - \min_{x \in \xset} \Fpop(x) \le 6GD \sqrt{\log d} \cdot \Par{\sqrt{\frac 1 {n }} + \frac{\sqrt{d\log \frac 1 {2\delta}}}{n\eps}}\end{equation}
		using
		\[O\Par{\frac{n^2\eps^2 d}{\log \frac 1 \delta} \log^2\Par{\frac{GDdn\eps}{\delta}}} \text{ value queries to some } f(\cdot; s_i).\]
	\end{enumerate}
\end{corollary}
\begin{proof}
We will parameterize Assumption~\ref{assume:norm} with the function $r(x) \defeq \frac {e^2}{2(q - 1)} \|x - x_0\|_q^2$, where $q = 1 + \frac 1 {\log d}$. By combining Proposition~\ref{prop:strongconvex} (which shows $r$ is $e^2$-strongly convex in $\ell_q$) and \eqref{eq:normequiv}, we have that $r$ is $1$-strongly convex in $\ell_1$. The remainder of the proof follows identically to Corollary~\ref{cor:pg1le2}.
\end{proof}

The term scaling as $\sqrt{\log d / n}$ in \eqref{eq:pge1le2sco}, namely the non-private population risk, is known to be optimal from existing lower bounds on SCO \cite{DuchiJW14}. Up to a $\sqrt{\log d}$ factor, the non-private empirical risk is optimal with respect to current private optimization lower bounds \cite{BassilyGN21, LiuL22}.

\begin{corollary}\label{cor:pge2}
	Let $p \ge 2$, and let $\eps > 0$, $\delta \in (0, 1)$. Let $\xset \subset \R^d$ have $\diam_{\norm{\cdot}_p}(\xset) \le D$, and let $\Fpop = \E_{s \sim \calP}[f(\cdot; s)]$ where all $f(\cdot; s): \R^d \to \R$ are convex and $G$-Lipschitz in $\norm{\cdot}_p$. Finally, let $\data = \{s_i\}_{i \in [n]} \sim \calP^n$ independently and $F_{\data} \defeq \frac 1 n \sum_{i \in [n]} f(\cdot; s_i)$. 
	\begin{enumerate}
		\item There is an $(\eps, \delta)$-differentially private algorithm $\mech$ which produces $x$ such that
		\begin{equation}\label{eq:pge2erm}\E_{\mech}\Brack{F_{\data}(x)} - \min_{x \in \xset} F_{\data}(x) \le 2GD \cdot \frac{d^{1- \frac 1 p}\sqrt{\log \frac 1 {2\delta}}}{n\eps }\end{equation}
		using
		\begin{equation}\label{eq:ge2N}O\Par{\frac{n^2\eps^2 d^{1 - \frac 2 p}}{\log \frac 1 \delta} \log^2\Par{\frac{GDdn\eps}{\delta}}} \text{ value queries to some } f(\cdot; s_i).\end{equation}
		\item There is an $(\eps, \delta)$-differentially private algorithm  $\mech$ which produces $x$ such that
		\begin{equation}\label{eq:pge2sco}\E_{\data \sim \calP^n, \mech}\Brack{\Fpop(x)} - \min_{x \in \xset} \Fpop(x) \le 2GD \cdot \Par{\frac{d^{\half - \frac 1 p}}{\sqrt n} + \frac{d^{1 - \frac 1 p}\sqrt{\log \frac 1 {2\delta}}}{n\eps}}.\end{equation}
		using
		\[O\Par{\frac{n^2\eps^2 d^{1 - \frac 2 p}}{\log \frac 1 \delta} \log^2\Par{\frac{GDdn\eps}{\delta}}} \text{ value queries to some } f(\cdot; s_i).\]
	\end{enumerate}
\end{corollary}
\begin{proof}
We will parameterize Assumption~\ref{assume:norm} with the function $r(x) \defeq \half \norm{x - x_0}_2^2$. By combining Proposition~\ref{prop:strongconvex} (which shows $r$ is $1$-strongly convex in $\ell_2$, and hence also $\ell_p$) and \eqref{eq:normequiv}, we may set $\Theta = \half d^{1 - 2/p} D^2$. The remainder of the proof follows identically to Corollary~\ref{cor:pg1le2}.
\end{proof}

\paragraph{Schatten-$p$ norms.} Our results extend immediately to matrix spaces equipped with Schatten-$p$ norm geometries. We record our relevant results in the following.

\begin{corollary}\label{cor:schatten}
Let $p \in [1, \infty)$, $\eps > 0$, $\delta \in (0, 1)$, and let $d_1, d_2 \in \N$ have $d_1 > d_2$. Let $\xset \subset \R^{d_1 \times d_2}$ have $\diam_{\norm{\cdot}_p}(\xset) \le D$, and let $\Fpop = \E_{s \sim \calP}[f(\cdot; s)]$ where all $f(\cdot; s): \R^{d_1 \times d_2} \to \R$ are convex and $G$-Lipschitz in $\norm{\cdot}_p$. Finally, let $\data = \{s_i\}_{i \in [n]} \sim \calP^n$ independently and $F_{\data} \defeq \frac 1 n \sum_{i \in [n]} f(\cdot; s_i)$.
\begin{enumerate}
	\item For constant $1 < p \le 2$, there is an $(\eps, \delta)$-differentially private algorithm $\mech$ which produces $\mm$ such that 
	\begin{gather*}
	\E_{\mech}[F_{\data}(\mm)] - \min_{\mm \in \xset} F_{\data}(\mm) \le 2GD \cdot \frac{\sqrt{d_1d_2 \log \frac 1 {2\delta}}}{n\eps\sqrt{p - 1}}, \\
	\E_{\data \sim \calP^n, \mech}[\Fpop(\mm)] - \min_{\mm \in \xset} \Fpop(\mm) \le 2GD \cdot \Par{\sqrt{\frac 1 {n(p - 1)}} + \frac{\sqrt{d_1d_2 \log \frac 1 {2\delta}}}{n\eps\sqrt{p - 1}}}.
	\end{gather*}
	The value oracle complexity of the algorithm is bounded as in \eqref{eq:g1le2N} for $d \gets d_2$ in the non-logarithmic term, and $d \gets d_1$ in the logarithmic term.
	\item For $p = 1$, there is an $(\eps, \delta)$-differentially private algorithm $\mech$ which produces $\mm$ such that 
	\begin{gather*}
	\E_{\mech}[F_{\data}(\mm)] - \min_{\mm \in \xset} F_{\data}(\mm) \le 6GD\sqrt{\log d_2} \cdot \frac{\sqrt{d_1d_2 \log \frac 1 {2\delta}}}{n\eps\sqrt{p - 1}},\\
	\E_{\data \sim \calP^n, \mech}[\Fpop(\mm)] - \min_{\mm \in \xset} \Fpop(\mm) \le 6GD\sqrt{\log d_2} \cdot \Par{\sqrt{\frac 1 n} + \frac{\sqrt{d_1d_2 \log \frac 1 {2\delta}}}{n\eps\sqrt{p - 1}}}.
	\end{gather*}
	The value oracle complexity of the algorithm is bounded as in \eqref{eq:1N} for $d \gets d_2$ in the non-logarithmic term, and $d \gets d_1$ in the logarithmic term.
	\item For $p \ge 2$, there is an $(\eps, \delta)$-differentially private algorithm $\mech$ which produces $\mm$ such that 
	\begin{gather*}
		\E_{\mech}[F_{\data}(\mm)] - \min_{\mm \in \xset} F_{\data}(\mm) \le 2GD \cdot \frac{d_2^{\half - \frac 1 p}\sqrt{d_1d_2 \log \frac 1 {2\delta}}}{n\eps},\\
		\E_{\data \sim \calP^n, \mech}[\Fpop(\mm)] - \min_{\mm \in \xset} \Fpop(\mm) \le 2GD \cdot \Par{\frac{d_2^{\half - \frac 1 p}}{\sqrt{n}} + \frac{d_2^{\half - \frac 1 p}\sqrt{d_1d_2 \log \frac 1 {2\delta}}}{n\eps}}.
	\end{gather*}
	The value oracle complexity of the algorithm is bounded as in \eqref{eq:ge2N} for $d \gets d_2$ in the non-logarithmic term, and $d \gets d_1$ in the logarithmic term.
\end{enumerate}
\end{corollary}
\begin{proof}
The privacy and utility proofs follow identically to Corollaries~\ref{cor:pg1le2},~\ref{cor:pge1le2}, and~\ref{cor:pge2}, where we use the second portion of Proposition~\ref{prop:strongconvex} instead of the first. We note that the ``dimension-dependent'' term in the risk inherited from Proposition~\ref{prop:gibbsrisk} scales as $d_1 d_2$ (the dimensionality of the matrix space). However, the terms in the risk due to the size of regularizers (inherited from the tradeoffs in \eqref{eq:normequiv}, for $p = 1$ and $p > 2$) scales as a power of $d_2$, the maximum dimension of singular values. To obtain the value oracle complexity, we note that by definition of the Schatten norm, it satisfies the relationship \eqref{eq:normequiv} as well. Moreover, the Schatten-$2$ norm agrees with the vector $\ell_2$ norm (when the matrix is flattened into a vector), since they are both the Frobenius norm. Hence, we may directly apply Proposition~\ref{prop:sample_elltwo} after paying a norm conversion, in the same way as was done in Proposition~\ref{prop:sample_ellp}.
\end{proof}

\paragraph{Remark on high-probability bounds.} One advantage of using a sampling-based algorithm is an immediate high-probability bound which follows due to the good concentration of Lipschitz functions over samples from strongly logconcave distributions, stated below.

\begin{lemma}[Concentration of Lipschitz functions, \cite{Ledoux99}, Section 2.3 and \cite{BobkovL00}, Proposition 3.1]
	\label{lm:concentration}
	Let $\ell$ be a $G$-Lipschitz function and $X\sim \exp(-F)$ for a $\mu$-strongly convex function $F$, all with respect to the same norm $\normx{\cdot}$. For all $t\ge 0$, 
	\begin{align*}
		\Pr\big[\ell(X)-\E[\ell(X)]\ge t\big]\le \exp\Par{-\frac{t^2\mu}{2G^2}}.
	\end{align*}
\end{lemma}

In particular, we have demonstrated that the population and empirical risks (which are Lipschitz) have good expectations.
Na\"ively combining Lemma~\ref{lm:concentration} and our main results on the expectation utility bound then yields tight concentration around the mean in some parameter regimes, but we suspect the resulting bound is loose in general. We leave it as an interesting open problem to obtain tight high-probability bounds in all parameter regimes.
	
	\subsection*{Acknowledgments} We would like to thank the anonymous SODA reviewers for their helpful suggestions. We would also like to thank a reviewer and Crist\'obal Guzm\'an for pointing out a mistake in the utility guarantees in Corollary~\ref{cor:schatten} in the original version, which is corrected in this version.
\newpage
	\bibliographystyle{alpha}	
	\bibliography{dp-general-norm}

\newcommand{\etalchar}[1]{$^{#1}$}
\begin{thebibliography}{CDWY20}

\bibitem[Abo16]{Abo16}
John~M. Abowd.
\newblock The challenge of scientific reproducibility and privacy protection
  for statistical agencies.
\newblock Technical report, Census Scientific Advisory Committee, 2016.

\bibitem[ABRW12]{AgarwalBRW12}
Alekh Agarwal, Peter~L. Bartlett, Pradeep Ravikumar, and Martin~J. Wainwright.
\newblock Information-theoretic lower bounds on the oracle complexity of
  stochastic convex optimization.
\newblock {\em {IEEE} Trans. Inf. Theory}, 58(5):3235--3249, 2012.

\bibitem[AC21]{AhnC21}
Kwangjun Ahn and Sinho Chewi.
\newblock Efficient constrained sampling via the mirror-langevin algorithm.
\newblock In Marc'Aurelio Ranzato, Alina Beygelzimer, Yann~N. Dauphin, Percy
  Liang, and Jennifer~Wortman Vaughan, editors, {\em Advances in Neural
  Information Processing Systems 34: Annual Conference on Neural Information
  Processing Systems 2021, NeurIPS 2021, December 6-14, 2021, virtual}, pages
  28405--28418, 2021.

\bibitem[ACG{\etalchar{+}}16]{AbadiCGMMT016}
Mart{\'{\i}}n Abadi, Andy Chu, Ian~J. Goodfellow, H.~Brendan McMahan, Ilya
  Mironov, Kunal Talwar, and Li~Zhang.
\newblock Deep learning with differential privacy.
\newblock In Edgar~R. Weippl, Stefan Katzenbeisser, Christopher Kruegel,
  Andrew~C. Myers, and Shai Halevi, editors, {\em Proceedings of the 2016 {ACM}
  {SIGSAC} Conference on Computer and Communications Security, Vienna, Austria,
  October 24-28, 2016}, pages 308--318. {ACM}, 2016.

\bibitem[AFKT21]{AsiFKT21}
Hilal Asi, Vitaly Feldman, Tomer Koren, and Kunal Talwar.
\newblock Private stochastic convex optimization: Optimal rates in {L1}
  geometry.
\newblock In Marina Meila and Tong Zhang, editors, {\em Proceedings of the 38th
  International Conference on Machine Learning, {ICML} 2021, 18-24 July 2021,
  Virtual Event}, volume 139 of {\em Proceedings of Machine Learning Research},
  pages 393--403. {PMLR}, 2021.

\bibitem[AGL{\etalchar{+}}18]{Allen-ZhuGLOW18}
Zeyuan Allen{-}Zhu, Ankit Garg, Yuanzhi Li, Rafael~Mendes de~Oliveira, and Avi
  Wigderson.
\newblock Operator scaling via geodesically convex optimization, invariant
  theory and polynomial identity testing.
\newblock In Ilias Diakonikolas, David Kempe, and Monika Henzinger, editors,
  {\em Proceedings of the 50th Annual {ACM} {SIGACT} Symposium on Theory of
  Computing, {STOC} 2018, Los Angeles, CA, USA, June 25-29, 2018}, pages
  172--181. {ACM}, 2018.

\bibitem[AHK12]{AroraHK12}
Sanjeev Arora, Elad Hazan, and Satyen Kale.
\newblock The multiplicative weights update method: a meta-algorithm and
  applications.
\newblock {\em Theory Comput.}, 8(1):121--164, 2012.

\bibitem[AKPS19]{AdilKPS19}
Deeksha Adil, Rasmus Kyng, Richard Peng, and Sushant Sachdeva.
\newblock Iterative refinement for {\(\mathscr{l}\)}p-norm regression.
\newblock In Timothy~M. Chan, editor, {\em Proceedings of the Thirtieth Annual
  {ACM-SIAM} Symposium on Discrete Algorithms, {SODA} 2019, San Diego,
  California, USA, January 6-9, 2019}, pages 1405--1424. {SIAM}, 2019.

\bibitem[ANW10]{AgarwalNW10}
Alekh Agarwal, Sahand~N. Negahban, and Martin~J. Wainwright.
\newblock Fast global convergence rates of gradient methods for
  high-dimensional statistical recovery.
\newblock In John~D. Lafferty, Christopher K.~I. Williams, John Shawe{-}Taylor,
  Richard~S. Zemel, and Aron Culotta, editors, {\em Advances in Neural
  Information Processing Systems 23: 24th Annual Conference on Neural
  Information Processing Systems 2010. Proceedings of a meeting held 6-9
  December 2010, Vancouver, British Columbia, Canada}, pages 37--45. Curran
  Associates, Inc., 2010.

\bibitem[BC12]{BubeckC12}
S{\'{e}}bastien Bubeck and Nicol{\`{o}} Cesa{-}Bianchi.
\newblock Regret analysis of stochastic and nonstochastic multi-armed bandit
  problems.
\newblock {\em Found. Trends Mach. Learn.}, 5(1):1--122, 2012.

\bibitem[BCL94]{BallCL94}
Keith Ball, Eric~A. Carlen, and Elliott~H. Lieb.
\newblock Sharp uniform convexity and smoothness estimates for trace norms.
\newblock {\em Inventiones mathematicae}, 115(1):463--482, 1994.

\bibitem[BDKT12]{BDKT12}
Aditya Bhaskara, Daniel Dadush, Ravishankar Krishnaswamy, and Kunal Talwar.
\newblock Unconditional differentially private mechanisms for linear queries.
\newblock In {\em Proceedings of the forty-fourth annual ACM symposium on
  Theory of computing}, pages 1269--1284, 2012.

\bibitem[BE02]{BousquetE02}
Olivier Bousquet and Andr{\'{e}} Elisseeff.
\newblock Stability and generalization.
\newblock {\em J. Mach. Learn. Res.}, 2:499--526, 2002.

\bibitem[BEM{\etalchar{+}}17]{BittauEMMRLRKTS17}
Andrea Bittau, {\'{U}}lfar Erlingsson, Petros Maniatis, Ilya Mironov, Ananth
  Raghunathan, David Lie, Mitch Rudominer, Ushasree Kode, Julien Tinn{\'{e}}s,
  and Bernhard Seefeld.
\newblock Prochlo: Strong privacy for analytics in the crowd.
\newblock In {\em Proceedings of the 26th Symposium on Operating Systems
  Principles, Shanghai, China, October 28-31, 2017}, pages 441--459. {ACM},
  2017.

\bibitem[BFGT20]{BassilyFGT20}
Raef Bassily, Vitaly Feldman, Crist{\'{o}}bal Guzm{\'{a}}n, and Kunal Talwar.
\newblock Stability of stochastic gradient descent on nonsmooth convex losses.
\newblock In Hugo Larochelle, Marc'Aurelio Ranzato, Raia Hadsell,
  Maria{-}Florina Balcan, and Hsuan{-}Tien Lin, editors, {\em Advances in
  Neural Information Processing Systems 33: Annual Conference on Neural
  Information Processing Systems 2020, NeurIPS 2020, December 6-12, 2020,
  virtual}, 2020.

\bibitem[BFTT19]{BassilyFTT19}
Raef Bassily, Vitaly Feldman, Kunal Talwar, and Abhradeep~Guha Thakurta.
\newblock Private stochastic convex optimization with optimal rates.
\newblock In Hanna~M. Wallach, Hugo Larochelle, Alina Beygelzimer, Florence
  d'Alch{\'{e}}{-}Buc, Emily~B. Fox, and Roman Garnett, editors, {\em Advances
  in Neural Information Processing Systems 32: Annual Conference on Neural
  Information Processing Systems 2019, NeurIPS 2019, December 8-14, 2019,
  Vancouver, BC, Canada}, pages 11279--11288, 2019.

\bibitem[BGM21]{BassilyGM21}
Raef Bassily, Crist{\'{o}}bal Guzm{\'{a}}n, and Michael Menart.
\newblock Differentially private stochastic optimization: New results in convex
  and non-convex settings.
\newblock In Marc'Aurelio Ranzato, Alina Beygelzimer, Yann~N. Dauphin, Percy
  Liang, and Jennifer~Wortman Vaughan, editors, {\em Advances in Neural
  Information Processing Systems 34: Annual Conference on Neural Information
  Processing Systems 2021, NeurIPS 2021, December 6-14, 2021, virtual}, pages
  9317--9329, 2021.

\bibitem[BGN21]{BassilyGN21}
Raef Bassily, Crist{\'{o}}bal Guzm{\'{a}}n, and Anupama Nandi.
\newblock Non-euclidean differentially private stochastic convex optimization.
\newblock In Mikhail Belkin and Samory Kpotufe, editors, {\em Conference on
  Learning Theory, {COLT} 2021, 15-19 August 2021, Boulder, Colorado, {USA}},
  volume 134 of {\em Proceedings of Machine Learning Research}, pages 474--499.
  {PMLR}, 2021.

\bibitem[BL00]{BobkovL00}
Sergey~G Bobkov and Michel Ledoux.
\newblock From brunn-minkowski to brascamp-lieb and to logarithmic sobolev
  inequalities.
\newblock {\em GAFA, Geometric and Functional Analysis}, 10:1028--1052, 2000.

\bibitem[Bor75a]{Borell75}
Christer Borell.
\newblock The brunn-minkowski in gauss space.
\newblock {\em Inventiones mathematicae}, 30:207--216, 1975.

\bibitem[Bor75b]{Bor75}
Christer Borell.
\newblock Convex set functions ind-space.
\newblock {\em Periodica Mathematica Hungarica}, 6(2):111--136, 1975.

\bibitem[BST14]{BassilyST14}
Raef Bassily, Adam~D. Smith, and Abhradeep Thakurta.
\newblock Private empirical risk minimization: Efficient algorithms and tight
  error bounds.
\newblock In {\em 55th {IEEE} Annual Symposium on Foundations of Computer
  Science, {FOCS} 2014, Philadelphia, PA, USA, October 18-21, 2014}, pages
  464--473. {IEEE} Computer Society, 2014.

\bibitem[Bub15]{Bubeck15}
S{\'{e}}bastien Bubeck.
\newblock Convex optimization: Algorithms and complexity.
\newblock {\em Found. Trends Mach. Learn.}, 8(3-4):231--357, 2015.

\bibitem[CDWY20]{ChenDWY20}
Yuansi Chen, Raaz Dwivedi, Martin~J. Wainwright, and Bin Yu.
\newblock Fast mixing of metropolized hamiltonian monte carlo: Benefits of
  multi-step gradients.
\newblock {\em J. Mach. Learn. Res.}, 21:92:1--92:72, 2020.

\bibitem[CE17]{Cordero17}
Dario Cordero-Erausquin.
\newblock Transport inequalities for log-concave measures, quantitative forms,
  and applications.
\newblock {\em Canada J. Math}, 69(3):481--501, 2017.

\bibitem[Che21]{Chen21}
Yuansi Chen.
\newblock An almost constant lower bound of the isoperimetric coefficient in
  the kls conjecture.
\newblock {\em GAFA, Geometric and Functional Analysis}, 31:34--61, 2021.

\bibitem[CM08]{ChaudhuriM08}
Kamalika Chaudhuri and Claire Monteleoni.
\newblock Privacy-preserving logistic regression.
\newblock In Daphne Koller, Dale Schuurmans, Yoshua Bengio, and L{\'{e}}on
  Bottou, editors, {\em Advances in Neural Information Processing Systems 21,
  Proceedings of the Twenty-Second Annual Conference on Neural Information
  Processing Systems, Vancouver, British Columbia, Canada, December 8-11,
  2008}, pages 289--296. Curran Associates, Inc., 2008.

\bibitem[CMS11]{ChaudhuriMS11}
Kamalika Chaudhuri, Claire Monteleoni, and Anand~D. Sarwate.
\newblock Differentially private empirical risk minimization.
\newblock {\em J. Mach. Learn. Res.}, 12:1069--1109, 2011.

\bibitem[CRT06]{CandesRT06}
Emmanuel~J. Cand{\`{e}}s, Justin~K. Romberg, and Terence Tao.
\newblock Robust uncertainty principles: exact signal reconstruction from
  highly incomplete frequency information.
\newblock {\em {IEEE} Trans. Inf. Theory}, 52(2):489--509, 2006.

\bibitem[DFO20]{DiakonikolasFO20}
Jelena Diakonikolas, Maryam Fazel, and Lorenzo Orecchia.
\newblock Fair packing and covering on a relative scale.
\newblock {\em {SIAM} J. Optim.}, 30(4):3284--3314, 2020.

\bibitem[DG21]{DiakonikolasG21}
Jelena Diakonikolas and Crist{\'{o}}bal Guzm{\'{a}}n.
\newblock Complementary composite minimization, small gradients in general
  norms, and applications to regression problems.
\newblock {\em CoRR}, abs/2101.11041, 2021.

\bibitem[DJW14]{DuchiJW14}
John~C. Duchi, Michael~I. Jordan, and Martin~J. Wainwright.
\newblock Privacy aware learning.
\newblock {\em J. {ACM}}, 61(6):38:1--38:57, 2014.

\bibitem[dKL18]{KlerkL18}
Etienne de~Klerk and Monique Laurent.
\newblock Comparison of lasserre's measure-based bounds for polynomial
  optimization to bounds obtained by simulated annealing.
\newblock {\em Math. Oper. Res.}, 43(4):1317--1325, 2018.

\bibitem[DKM{\etalchar{+}}06]{DworkKMMN06}
Cynthia Dwork, Krishnaram Kenthapadi, Frank McSherry, Ilya Mironov, and Moni
  Naor.
\newblock Our data, ourselves: Privacy via distributed noise generation.
\newblock In Serge Vaudenay, editor, {\em Advances in Cryptology - {EUROCRYPT}
  2006, 25th Annual International Conference on the Theory and Applications of
  Cryptographic Techniques, St. Petersburg, Russia, May 28 - June 1, 2006,
  Proceedings}, volume 4004 of {\em Lecture Notes in Computer Science}, pages
  486--503. Springer, 2006.

\bibitem[DKY17]{DingKY17}
Bolin Ding, Janardhan Kulkarni, and Sergey Yekhanin.
\newblock Collecting telemetry data privately.
\newblock In Isabelle Guyon, Ulrike von Luxburg, Samy Bengio, Hanna~M. Wallach,
  Rob Fergus, S.~V.~N. Vishwanathan, and Roman Garnett, editors, {\em Advances
  in Neural Information Processing Systems 30: Annual Conference on Neural
  Information Processing Systems 2017, December 4-9, 2017, Long Beach, CA,
  {USA}}, pages 3571--3580, 2017.

\bibitem[DMNS06]{DworkMNS06}
Cynthia Dwork, Frank McSherry, Kobbi Nissim, and Adam~D. Smith.
\newblock Calibrating noise to sensitivity in private data analysis.
\newblock In Shai Halevi and Tal Rabin, editors, {\em Theory of Cryptography,
  Third Theory of Cryptography Conference, {TCC} 2006, New York, NY, USA, March
  4-7, 2006, Proceedings}, volume 3876 of {\em Lecture Notes in Computer
  Science}, pages 265--284. Springer, 2006.

\bibitem[DRS21]{DongRS21}
Jinshuo Dong, Aaron Roth, and Weijie~J. Su.
\newblock Gaussian differential privacy.
\newblock {\em Journal of the Royal Statistical Society: Series B (Statistical
  Methodology)}, 84(1):3--37, 2021.

\bibitem[EPK14]{ErlingssonPK14}
{\'{U}}lfar Erlingsson, Vasyl Pihur, and Aleksandra Korolova.
\newblock {RAPPOR:} randomized aggregatable privacy-preserving ordinal
  response.
\newblock In Gail{-}Joon Ahn, Moti Yung, and Ninghui Li, editors, {\em
  Proceedings of the 2014 {ACM} {SIGSAC} Conference on Computer and
  Communications Security, Scottsdale, AZ, USA, November 3-7, 2014}, pages
  1054--1067. {ACM}, 2014.

\bibitem[FG04]{FG04}
Matthieu Fradelizi and Olivier Gu{\'e}don.
\newblock The extreme points of subsets of s-concave probabilities and a
  geometric localization theorem.
\newblock {\em Discrete \& Computational Geometry}, 31(2):327--335, 2004.

\bibitem[FKT20]{FeldmanKT20}
Vitaly Feldman, Tomer Koren, and Kunal Talwar.
\newblock Private stochastic convex optimization: optimal rates in linear time.
\newblock In Konstantin Makarychev, Yury Makarychev, Madhur Tulsiani, Gautam
  Kamath, and Julia Chuzhoy, editors, {\em Proccedings of the 52nd Annual {ACM}
  {SIGACT} Symposium on Theory of Computing, {STOC} 2020, Chicago, IL, USA,
  June 22-26, 2020}, pages 439--449. {ACM}, 2020.

\bibitem[GLL22]{GLL22}
Sivakanth Gopi, Yin~Tat Lee, and Daogao Liu.
\newblock Private convex optimization via exponential mechanism.
\newblock In Po{-}Ling Loh and Maxim Raginsky, editors, {\em Conference on
  Learning Theory, 2-5 July 2022, London, {UK}}, volume 178 of {\em Proceedings
  of Machine Learning Research}, pages 1948--1989. {PMLR}, 2022.

\bibitem[GTU22]{GaneshTU22}
Arun Ganesh, Abhradeep Thakurta, and Jalaj Upadhyay.
\newblock Langevin diffusion: An almost universal algorithm for private
  euclidean (convex) optimization.
\newblock {\em CoRR}, abs/2204.01585, 2022.

\bibitem[HKRC18]{HsiehKRC18}
Ya{-}Ping Hsieh, Ali Kavis, Paul Rolland, and Volkan Cevher.
\newblock Mirrored langevin dynamics.
\newblock In Samy Bengio, Hanna~M. Wallach, Hugo Larochelle, Kristen Grauman,
  Nicol{\`{o}} Cesa{-}Bianchi, and Roman Garnett, editors, {\em Advances in
  Neural Information Processing Systems 31: Annual Conference on Neural
  Information Processing Systems 2018, NeurIPS 2018, December 3-8, 2018,
  Montr{\'{e}}al, Canada}, pages 2883--2892, 2018.

\bibitem[HLL{\etalchar{+}}22]{HLL+22}
Yuxuan Han, Zhicong Liang, Zhipeng Liang, Yang Wang, Yuan Yao, and Jiheng
  Zhang.
\newblock Private streaming sco in $ell\_p $ geometry with applications in high
  dimensional online decision making.
\newblock In {\em International Conference on Machine Learning}, pages
  8249--8279. PMLR, 2022.

\bibitem[HT10]{HT10}
Moritz Hardt and Kunal Talwar.
\newblock On the geometry of differential privacy.
\newblock In {\em Proceedings of the forty-second ACM symposium on Theory of
  computing}, pages 705--714, 2010.

\bibitem[Jia21]{Jiang21}
Qijia Jiang.
\newblock Mirror langevin monte carlo: the case under isoperimetry.
\newblock In Marc'Aurelio Ranzato, Alina Beygelzimer, Yann~N. Dauphin, Percy
  Liang, and Jennifer~Wortman Vaughan, editors, {\em Advances in Neural
  Information Processing Systems 34: Annual Conference on Neural Information
  Processing Systems 2021, NeurIPS 2021, December 6-14, 2021, virtual}, pages
  715--725, 2021.

\bibitem[JLLV21]{JiaLLV21}
He~Jia, Aditi Laddha, Yin~Tat Lee, and Santosh~S. Vempala.
\newblock Reducing isotropy and volume to {KLS:} an
  \emph{o}*(\emph{n}\({}^{\mbox{3}}\)\emph{{\(\psi\)}}\({}^{\mbox{2}}\)) volume
  algorithm.
\newblock In Samir Khuller and Virginia~Vassilevska Williams, editors, {\em
  {STOC} '21: 53rd Annual {ACM} {SIGACT} Symposium on Theory of Computing,
  Virtual Event, Italy, June 21-25, 2021}, pages 961--974. {ACM}, 2021.

\bibitem[JLT20]{Jambulapati0T20}
Arun Jambulapati, Jerry Li, and Kevin Tian.
\newblock Robust sub-gaussian principal component analysis and
  width-independent schatten packing.
\newblock In Hugo Larochelle, Marc'Aurelio Ranzato, Raia Hadsell,
  Maria{-}Florina Balcan, and Hsuan{-}Tien Lin, editors, {\em Advances in
  Neural Information Processing Systems 33: Annual Conference on Neural
  Information Processing Systems 2020, NeurIPS 2020, December 6-12, 2020,
  virtual}, 2020.

\bibitem[JT14]{JainT14}
Prateek Jain and Abhradeep~Guha Thakurta.
\newblock (near) dimension independent risk bounds for differentially private
  learning.
\newblock In {\em Proceedings of the 31th International Conference on Machine
  Learning, {ICML} 2014, Beijing, China, 21-26 June 2014}, volume~32 of {\em
  {JMLR} Workshop and Conference Proceedings}, pages 476--484. JMLR.org, 2014.

\bibitem[KJ16]{Kasiviswanathan16}
Shiva~Prasad Kasiviswanathan and Hongxia Jin.
\newblock Efficient private empirical risk minimization for high-dimensional
  learning.
\newblock In Maria{-}Florina Balcan and Kilian~Q. Weinberger, editors, {\em
  Proceedings of the 33nd International Conference on Machine Learning, {ICML}
  2016, New York City, NY, USA, June 19-24, 2016}, volume~48 of {\em {JMLR}
  Workshop and Conference Proceedings}, pages 488--497. JMLR.org, 2016.

\bibitem[KL22]{KlartagL22}
Bo'az Klartag and Joseph Lehec.
\newblock Bourgain's slicing problem and kls isoperimetry up to polylog.
\newblock {\em arXiv preprint arXiv:2203.15551}, 2022.

\bibitem[KLL21]{KulkarniLL21}
Janardhan Kulkarni, Yin~Tat Lee, and Daogao Liu.
\newblock Private non-smooth {ERM} and {SCO} in subquadratic steps.
\newblock In Marc'Aurelio Ranzato, Alina Beygelzimer, Yann~N. Dauphin, Percy
  Liang, and Jennifer~Wortman Vaughan, editors, {\em Advances in Neural
  Information Processing Systems 34: Annual Conference on Neural Information
  Processing Systems 2021, NeurIPS 2021, December 6-14, 2021, virtual}, pages
  4053--4064, 2021.

\bibitem[KLOS14]{KelnerLOS14}
Jonathan~A. Kelner, Yin~Tat Lee, Lorenzo Orecchia, and Aaron Sidford.
\newblock An almost-linear-time algorithm for approximate max flow in
  undirected graphs, and its multicommodity generalizations.
\newblock In Chandra Chekuri, editor, {\em Proceedings of the Twenty-Fifth
  Annual {ACM-SIAM} Symposium on Discrete Algorithms, {SODA} 2014, Portland,
  Oregon, USA, January 5-7, 2014}, pages 217--226. {SIAM}, 2014.

\bibitem[KLS95]{kannan1995isoperimetric}
Ravi Kannan, L{\'a}szl{\'o} Lov{\'a}sz, and Mikl{\'o}s Simonovits.
\newblock Isoperimetric problems for convex bodies and a localization lemma.
\newblock {\em Discrete \& Computational Geometry}, 13(3):541--559, 1995.

\bibitem[Kol11]{Kolesnikov11}
Alexander~V. Kolesnikov.
\newblock Mass transportation and contractions.
\newblock {\em arXiv preprint arXiv:1103.1479}, 2011.

\bibitem[KPSW19]{KyngPSW19}
Rasmus Kyng, Richard Peng, Sushant Sachdeva, and Di~Wang.
\newblock Flows in almost linear time via adaptive preconditioning.
\newblock In Moses Charikar and Edith Cohen, editors, {\em Proceedings of the
  51st Annual {ACM} {SIGACT} Symposium on Theory of Computing, {STOC} 2019,
  Phoenix, AZ, USA, June 23-26, 2019}, pages 902--913. {ACM}, 2019.

\bibitem[KST12]{KiferST12}
Daniel Kifer, Adam~D. Smith, and Abhradeep Thakurta.
\newblock Private convex optimization for empirical risk minimization with
  applications to high-dimensional regression.
\newblock In Shie Mannor, Nathan Srebro, and Robert~C. Williamson, editors,
  {\em {COLT} 2012 - The 25th Annual Conference on Learning Theory, June 25-27,
  2012, Edinburgh, Scotland}, volume~23 of {\em {JMLR} Proceedings}, pages
  25.1--25.40. JMLR.org, 2012.

\bibitem[KV06]{KalaiV06}
Adam~Tauman Kalai and Santosh~S. Vempala.
\newblock Simulated annealing for convex optimization.
\newblock {\em Math. Oper. Res.}, 31(2):253--266, 2006.

\bibitem[Led99]{Ledoux99}
Michel Ledoux.
\newblock {\em Concentration of measure and logarithmic Sobolev inequalities}.
\newblock Seminaire de probabilities XXXIII, 1999.

\bibitem[LL22]{LiuL22}
Daogao Liu and Zhou Lu.
\newblock Lower bounds for differentially private erm: Unconstrained and
  non-euclidean.
\newblock {\em arXiv preprint arXiv:2105.13637}, 2022.

\bibitem[LS93]{lovasz1993random}
L{\'a}szl{\'o} Lov{\'a}sz and Mikl{\'o}s Simonovits.
\newblock Random walks in a convex body and an improved volume algorithm.
\newblock {\em Random structures \& algorithms}, 4(4):359--412, 1993.

\bibitem[LST21]{lee2021structured}
Yin~Tat Lee, Ruoqi Shen, and Kevin Tian.
\newblock Structured logconcave sampling with a restricted gaussian oracle.
\newblock In {\em Conference on Learning Theory}, pages 2993--3050. PMLR, 2021.

\bibitem[LTVW22]{LiTVW22}
Ruilin Li, Molei Tao, Santosh~S. Vempala, and Andre Wibisono.
\newblock The mirror langevin algorithm converges with vanishing bias.
\newblock In Sanjoy Dasgupta and Nika Haghtalab, editors, {\em International
  Conference on Algorithmic Learning Theory, 29-1 April 2022, Paris, France},
  volume 167 of {\em Proceedings of Machine Learning Research}, pages 718--742.
  {PMLR}, 2022.

\bibitem[LV07]{LovaszV07}
L{\'{a}}szl{\'{o}} Lov{\'{a}}sz and Santosh~S. Vempala.
\newblock The geometry of logconcave functions and sampling algorithms.
\newblock {\em Random Struct. Algorithms}, 30(3):307--358, 2007.

\bibitem[MS08]{MilmanS08}
Emanuel Milman and Sasha Sodin.
\newblock An isoperimetric inequality for uniformly log-concave measures and
  uniformly convex bodies.
\newblock {\em Journal of Functional Analysis}, 254(5):1235--1268, 2008.

\bibitem[MT07]{McSherryT07}
Frank McSherry and Kunal Talwar.
\newblock Mechanism design via differential privacy.
\newblock In {\em 48th Annual {IEEE} Symposium on Foundations of Computer
  Science {(FOCS} 2007), October 20-23, 2007, Providence, RI, USA,
  Proceedings}, pages 94--103. {IEEE} Computer Society, 2007.

\bibitem[Nem04]{nemirovski2004interior}
Arkadi Nemirovski.
\newblock Interior point polynomial time methods in convex programming.
\newblock {\em Lecture notes}, 42(16):3215--3224, 2004.

\bibitem[NY83]{NemirovskiY83}
A.\ Nemirovski and D.\~B.\ Yudin.
\newblock {\em Problem Complexity and Method Efficiency in Optimization}.
\newblock Wiley, 1983.

\bibitem[OV00]{otto2000generalization}
Felix Otto and C{\'e}dric Villani.
\newblock Generalization of an inequality by talagrand and links with the
  logarithmic sobolev inequality.
\newblock {\em Journal of Functional Analysis}, 173(2):361--400, 2000.

\bibitem[Sha12]{Shalev-Shwartz12}
Shai Shalev{-}Shwartz.
\newblock Online learning and online convex optimization.
\newblock {\em Found. Trends Mach. Learn.}, 4(2):107--194, 2012.

\bibitem[ST74]{SudakovT74}
Vladimir Sudakov and Boris Tsirelson.
\newblock Extremal properties of half-spaces for spherically invariant
  measures.
\newblock {\em J. Soviet Math}, 9:9--18, 1974.

\bibitem[Tea17]{App17}
Apple Differential~Privacy Team.
\newblock Learning with privacy at scale.
\newblock Technical report, Apple, 2017.

\bibitem[ZPFP20]{ZhangPFP20}
Kelvin~Shuangjian Zhang, Gabriel Peyr{\'{e}}, Jalal Fadili, and Marcelo
  Pereyra.
\newblock Wasserstein control of mirror langevin monte carlo.
\newblock In Jacob~D. Abernethy and Shivani Agarwal, editors, {\em Conference
  on Learning Theory, {COLT} 2020, 9-12 July 2020, Virtual Event [Graz,
  Austria]}, volume 125 of {\em Proceedings of Machine Learning Research},
  pages 3814--3841. {PMLR}, 2020.

\end{thebibliography}
	
	\begin{appendix}
\section{Private ERM and SCO under strong convexity}\label{app:sc}

In this section, we derive our results for private ERM and SCO in general norms under the assumption that the sample losses are strongly convex. We will state our results for private ERM (Theorem~\ref{thm:erm_sc}) and SCO (Theorem~\ref{thm:sco_sc}) with respect to an arbitrary compact convex subset $\xset$ of a $d$-dimensional normed space, satisfying the following Assumption~\ref{assume:norm_sc}. 
\begin{assumption}\label{assume:norm_sc}
	We make the following assumptions.
	\begin{enumerate}
		\item There is a compact, convex subspace $\xset \subset \R^d$ equipped with a norm $\normx{\cdot}$.
		\item There is a set $\Omega$ such that for any $s \in \Omega$, there is a function $f(\cdot; s): \xset \to \R$ which is $G$-Lipschitz and $\mu$-strongly convex in $\normx{\cdot}$.
	\end{enumerate}
\end{assumption}

\begin{theorem}[Private ERM]
	\label{thm:erm_sc}
	Under Assumption~\ref{assume:norm_sc} and following notation \eqref{eq:Fdef}, drawing a sample $x$ from the density $\nu \propto \exp (-k F_{\data})$ for 
	$$k = \frac {n^2 \eps^2\mu}{2G^2\log \frac 1 {2\delta}},$$
	is $(\epsilon, \delta )$-differentially private, and produces $x$ such that 
		\[\E_{x \sim \nu}[F_{\data}(x)] - \min_{x \in \xset} F_{\data}(x) \le  \frac{2dG^2\log \frac 1 {2\delta}}{n^2\epsilon^2 \mu}.\]
\end{theorem}
\begin{proof}
	Let $F_{\data'}$ be the realization of \eqref{eq:Fdef} when $\data$ is replaced with a neighboring dataset $\data'$ which agrees in all entries except some sample $s'_i \neq s_i$. By Assumption~\ref{assume:norm_sc}, we have $k(F_{\data} - F_{\data'})$ is $\frac{kG}{n}$-Lipschitz, and both $kF_{\data} $ and $kF_{\data'} $ are $k\mu$-strongly convex (all with respect to $\normx{\cdot}$). Hence, combining Theorem~\ref{thm:privacy_curve_norm_X} and Fact~\ref{fact:gprivacy} shows the mechanism is $(\eps, \delta)$-differentially private, since
	\[k = \frac {n^2 \eps^2\mu}{2G^2\log \frac 1 {2\delta}}\implies \frac{G\sqrt k}{n\sqrt \mu} \le \frac \eps {\sqrt{2 \log \frac 1 {2\delta}}}.\]
	Let $x^\star_{\data} \defeq \argmin_{x \in \xset} F_{\data}(x)$. We obtain the risk guarantee by the calculation (see Proposition~\ref{prop:gibbsrisk})
	\begin{align*}
		\E_{x \sim \nu}[F_{\data}(x)] &\le F_{\data}(x^\star_{\data}) + \frac d k 
	\le F_{\data}(x^\star_{\data}) + \frac{2dG^2\log \frac 1 {2\delta}}{n^2\epsilon^2 \mu}.
	\end{align*}
\end{proof}

\begin{theorem}[Private SCO]\label{thm:sco_sc}
	Under Assumption~\ref{assume:norm_sc} and following notation \eqref{eq:Fdef}, drawing a sample $x$ from the density $\nu \propto \exp(-kF_{\data} )$ for
	\[k = \frac {n^2 \eps^2\mu}{2G^2\log \frac 1 {2\delta}}\]
	is $(\eps, \delta)$-differentially private, and produces $x$ such that 
	\[\E_{\data \sim \pi^n, x \sim \nu}\Brack{\Fpop(x)} - \min_{x \in \xset} \Fpop(x) \le \frac{G^2}{n\mu }\Par{1+ \frac{2d\log \frac 1 {2\delta}}{n\epsilon^2 }} .\]
\end{theorem}
\begin{proof}
	For the given choice $k, \mu$, the privacy proof follows identically to Theorem~\ref{thm:erm}, so we focus on the risk proof. We follow the notation of Proposition~\ref{prop:poprisk} and let $s \sim \pi$ independently from $\pi$. By exchanging the expectation and minimum and using that $\E_{\data \sim \pi^n} F_{\data} = \Fpop$,
	\begin{align*}\E_{\data \sim \pi^n, x \sim \nu}\Brack{\Fpop(x)} - \min_{x \sim \xset} \Fpop(x) &\le \E_{\data \sim \pi^n}\Brack{\E_{x \sim \nu}\Brack{\Fpop(x)} - \min_{x \in \xset}\Brack{F_{\data}(x)}} \\
		&\le \E_{\data \sim \pi^n}\Brack{\E_{x \sim \nu}\Brack{\Fpop(x) - F_{\data}(x)}} \\
		&+ \E_{\data \sim \pi^n}\Brack{\E_{x \sim \nu}\Brack{F_{\data}(x)} - \min_{x \in \xset}\Brack{F_{\data}(x)}} \\
		&\le \E_{\data \sim \pi^n}\Brack{\E_{x \sim \nu}\Brack{\Fpop(x) - F_{\data}(x)}} + \frac d k,
	\end{align*}
	where we bounded the empirical risk in the proof of Theorem~\ref{thm:erm_sc}. Next, let $\nu'$ be the density $\propto \exp(-kF_{\data'} )$. Our mechanism is symmetric, and hence by Proposition~\ref{prop:poprisk},
	\[
	\E\Brack{\Fpop(x) - F_{\data}(x)} = \E\Brack{\E_{x \sim \nu}\Brack{f(x; s)} - \E_{x \sim \nu'}\Brack{f(x; s)}}
	\]
	where the outer expectations are over the randomness of drawing $\data, s$. Finally, for any fixed realization of $\data, s$, the densities $\nu, \nu'$ satisfy the assumption of Corollary~\ref{cor:liplocalization} with $H = \frac G n$, and $f(\cdot; s)$ is $G$-Lipschitz, so Corollary~\ref{cor:liplocalization} shows that
	\[\E_{x \sim \nu}\Brack{f(x; s)} - \E_{x \sim \nu'}\Brack{f(x; s)} \le \frac{G^2}{n\mu}.\]
	Combining the above three displays bounds the population risk by
	\begin{align*}\E_{\data \sim \pi^n, x \sim \nu}\Brack{\Fpop(x)} - \min_{x \in \xset} \Fpop(x) &\le \frac{G^2}{n\mu} +  \frac d k 
		= \frac{G^2}{n\mu }\Par{1+ \frac{2d\log \frac 1 {2\delta}}{n\epsilon^2 }} .
	\end{align*}
	for our given value of $k$. 
\end{proof}
	\end{appendix}

\end{document}